\newcommand{\arxiv}[1]{\ignorespaces}
\renewcommand{\arxiv}[1]{#1}
\newcommand{\jmlr}[1]{\ignorespaces}

\arxiv{
  \documentclass[12pt]{article}
  \usepackage[letterpaper,margin=1in]{geometry}
  \title{The Sample Complexity of Parameter-Free Stochastic Convex Optimization}
   \author{
    Jared Lawrence\thanks{Department of Industrial Engineering, University of Pittsburgh.}
    \and
    Ari Kalinsky\footnotemark[1]
    \and
    Hannah Bradfield\footnotemark[1]
    \and
    Yair Carmon\thanks{Department of Computer Science, Tel Aviv University.}
    \and
    Oliver Hinder\footnotemark[1] \thanks{Corresponding author: \texttt{ohinder@pitt.edu}.}
  }
}

\jmlr{
  \documentclass[twoside,11pt]{article}
  \usepackage{jmlr2e}
  
  \jmlrheading{27}{2026}{1-\pageref{LastPage}}{6/25; Revised 3/26}{5/26}{25-2383}{Jared Lawrence, Ari Kalinsky, Hannah Bradfield, Yair Carmon, and Oliver Hinder}
  \ShortHeadings{The Sample Complexity Of Parameter-Free Stochastic Convex Optimization}{Lawrence, Kalinsky, Bradfield, Carmon, and Hinder}
  \firstpageno{1}

  \title{The Sample Complexity of Parameter-Free Stochastic Convex Optimization}
    \author{\name Jared Lawrence \email JPL86@pitt.edu \\
    \addr Department of Industrial Engineering, University of Pittsburgh \\
    Pittsburgh, PA 15261, USA
    \AND
    \name Ari Kalinsky \email ajk245@pitt.edu \\
    \addr Department of Industrial Engineering, University of Pittsburgh \\
    Pittsburgh, PA 15261, USA
    \AND
    \name Hannah Bradfield \email hmb78@pitt.edu \\
    \addr Department of Industrial Engineering, University of Pittsburgh \\
    Pittsburgh, PA 15261, USA
    \AND
    \name Yair Carmon \email ycarmon@tauex.tau.ac.il \\
    \addr Department of Computer Science, Tel Aviv University \\
    Tel Aviv 6997801, Israel
    \AND
    \name Oliver Hinder \email ohinder@pitt.edu \\
    \addr Department of Industrial Engineering, University of Pittsburgh \\
    Pittsburgh, PA 15261, USA
    }

  \editor{Shiqian Ma}
  
  \date{}
}

\usepackage{physics}
\usepackage{hyperref}
\arxiv{
\hypersetup{
	colorlinks=true,      %
	linkcolor=blue,       %
	citecolor=[rgb]{0,0.2,0},      %
	urlcolor=black      %
}
}

\jmlr{
	\hypersetup{hidelinks}
}

\usepackage{algorithm}
\usepackage{algorithmic}
\usepackage{mathtools}

\usepackage{longtable}
\usepackage{array}
\usepackage{tabularx}
\usepackage{makecell}
\usepackage{booktabs}
\usepackage{enumitem}
\setlist[enumerate]{leftmargin=0.8cm}

\arxiv{
\usepackage{xcolor}
\usepackage[numbers,sort,compress]{natbib}
}

\arxiv{
	\usepackage{amssymb}
	\usepackage{amsthm}
	\newtheorem{proposition}{Proposition}
	\newtheorem{lemma}{Lemma}
	\newtheorem{assumption}{Assumption}
	\newtheorem{theorem}{Theorem}
	\newtheorem{corollary}{Corollary}
}
\jmlr{
	\newtheorem{assumption}{Assumption}
}

\usepackage{thmtools}

\usepackage[capitalize,noabbrev]{cleveref}

\jmlr{
	
	\makeatletter

	\makeatother

	\newtheorem{lemma}[theorem]{Lemma}
	\newtheorem{proposition}[theorem]{Proposition}
	\newtheorem{remark}[theorem]{Remark}
	\newtheorem{corollary}[theorem]{Corollary}
	\newtheorem{definition}[theorem]{Definition}
	\newtheorem{conjecture}[theorem]{Conjecture}
	\newtheorem{axiom}[theorem]{Axiom}

	\crefname{theorem}{Theorem}{Theorems}
	\Crefname{theorem}{Theorem}{Theorems}
	\crefname{lemma}{Lemma}{Lemmas}
	\Crefname{lemma}{Lemma}{Lemmas}
	\crefname{proposition}{Proposition}{Propositions}
	\Crefname{proposition}{Proposition}{Propositions}
	\crefname{remark}{Remark}{Remarks}
	\Crefname{remark}{Remark}{Remarks}
	\crefname{corollary}{Corollary}{Corollaries}
	\Crefname{corollary}{Corollary}{Corollaries}
	\crefname{definition}{Definition}{Definitions}
	\Crefname{definition}{Definition}{Definitions}
	\crefname{conjecture}{Conjecture}{Conjectures}
	\Crefname{conjecture}{Conjecture}{Conjectures}
	\crefname{axiom}{Axiom}{Axioms}
	\Crefname{axiom}{Axiom}{Axioms}
}

\jmlr{
	
	\renewcommand{\cite}{\citep}
}

\usepackage{tikz}
\usepackage{pgfplots}
\pgfplotsset{compat=1.17}
\usepackage{xcolor}

\newcommand{\LipCoord}[0]{\mathbf{L}}
\newcommand{\Lip}[0]{L}
\newcommand{\LipHat}[0]{\hat{L}}

\newcommand{\ollie}[1]{{\color{blue}OH: #1}}
\newcommand{\yair}[1]{{\color{blue}#1}}

\renewcommand{\ollie}[1]{\ignorespaces}
\renewcommand{\yair}[1]{\ignorespaces}

\newcommand{\RMS}[0]{\textsc{ReliableModelSelection}}
\newcommand{\OutputPerfect}[0]{\edit{x^{\textup{perfect}}_{n,p}}}

\newcommand{\lambdaAlg}[1][]{%
  \if\relax\detokenize{#1}\relax
    \edit{\phi_n}%
  \else
    \edit{\phi_{n,#1}}%
  \fi
}
\let\oldlambda\lambda
\renewcommand{\lambda}[1][]{%
  \ifstrempty{#1}%
    {\edit{\oldlambda_n}}%
    {\edit{\oldlambda_{n,#1}}}%
}
\newcommand{\hatlambda}[1][]{%
  \ifstrempty{#1}%
    {\edit{\hat{\oldlambda}_n}}%
    {\edit{\hat{\oldlambda}_{n,#1}}}%
}

\DeclareMathOperator*{\argmin}{arg\,min}

\newcommand{\zeros}[0]{\mathbf{0}}

\renewcommand{\P}{\mathbb{P}}
\newcommand{\E}{\mathbb{E}}
\newcommand{\R}{\mathbb{R}}

\arxiv{
\newcommand{\myparagraph}[1]{\paragraph{#1}}
}
\jmlr{
  \newcommand{\myparagraph}[1]{\medskip\noindent\textsc{#1.}\ } %
}

\newcommand{\mysection}[2]{\section{#1}}
\newcommand{\mysubsection}[2]{\subsection{#1}}

\jmlr{
  \renewcommand{\mysection}[2]{\section{#2}}
  \renewcommand{\mysubsection}[2]{\subsection{#2}}
  
}

\newcommand{\email}[1]{\\[-4pt]%
	{\small\href{mailto:\detokenize{#1}}{\texttt{#1}}}%
}

\newtheorem{condition}{Condition}

\newcommand{\proj}{\mathbf{proj}}

\newcommand{\alg}{\mathbf{A}}

\newcommand{\DStar}[0]{R^\star}

\renewcommand{\SS}[0]{\mathcal{S}}
\newcommand{\kGreedy}[0]{k_{\textup{std}}}

\newcommand{\AdaGrad}[1]{\textsc{AdaGrad}(#1)}
\newcommand{\AdaSGD}[1]{\textsc{AdaSGD}(#1)}
\newcommand{\SimplexMirrorDescent}[1]{\textsc{AdaEMD}(#1)}
\newcommand{\xReg}[1]{\hat{x}_{#1}}
\newcommand{\xRegStar}[1]{x_{#1}^\star}
\newcommand{\SampleVar}[1]{\mathbb{V}_{#1}}
\newcommand{\kReliable}[0]{k_{\textup{rely}}}

\newcommand{\coor}[2][j]{[#2]_{#1}} %
\newcommand{\X}{\mathcal{X}}

\newcommand{\us}[2]{\underset{(#2)}{#1}}
\arxiv{\newcommand{\LongPaper}[1]{#1}}
\arxiv{\newcommand{\ShortPaper}[1]{}}
\jmlr{\newcommand{\LongPaper}[1]{#1}}
\jmlr{\newcommand{\ShortPaper}[1]{}}

\newcommand{\e}{e}
\newcommand{\XStar}[0]{\mathcal{X}^\star}
\newcommand{\xstar}[0]{x^\star}
\newcommand{\Fstar}[0]{F^\star}

\newcommand{\barF}[1]{\bar{F}(x_{#1})}
\newcommand{\kStar}[0]{k_\star}

\newcommand{\GreedyPoorSetup}[0]{\inf_{x \in \X} \E[ f(x; S )]  \text{ where } f(x; 0) = \abs{x},  f(x; 1) = -x}

\newcommand{\Vub}{C}

\newcommand{\gradDiff}{\Delta_j(x)}

\newcommand{\edit}[1]{#1}
\usepackage{lastpage}

\begin{document}

\maketitle

\jmlr{
    \begin{abstract}%
}
\arxiv{\begin{abstract}}
We study the sample complexity of stochastic convex optimization when problem parameters \edit{such as the distance to optimality and the Lipschitz constant} are unknown. We pursue two strategies. First, we develop a reliable model selection method that avoids overfitting \edit{to }the validation set. This method allows us to generically tune the learning rate of stochastic optimization methods 
to match the optimal known-parameter sample complexity up to $\log\log$ factors.
Second, we develop a regularization-based method that is specialized to the case that only the distance to optimality is unknown. \edit{More specifically, it uses norm-regularized empirical risk minimization to estimate the distance to optimality to within a constant factor, allowing known-parameter stochastic optimization methods to achieve optimal sample complexity.} This method provides perfect adaptability to unknown distance to optimality, demonstrating a separation between the sample and computational complexity of parameter-free stochastic convex optimization.
Combining these two methods allows us to simultaneously adapt to multiple problem structures.

Experiments performing few-shot learning on CIFAR-10 by fine-tuning CLIP models and prompt engineering Gemini to count shapes indicate that our reliable model selection method can help mitigate overfitting to small validation sets.
\end{abstract}

\jmlr{
}

\section{Introduction}
 
In machine learning, there is a tension between computational and sample efficiency.
If we have unlimited data, then computational efficiency dominates our concerns.
For example, when training large language models from scratch, it is typical to only make one pass on the training data with a single set of hyperparameters \cite{kaplan2020scaling}.
On the other hand, if we have limited data, then methods that efficiently use that data are paramount and more computationally intensive approaches become viable.
For example, in few-shot learning, it is standard to sweep over many hyperparameters and make multiple data passes \cite{hu2022pushing}.

The classical theory of stochastic convex optimization assumes that the Lipschitz constant $L$ and distance $\DStar$ from the initial point to the optimum are known. In this setting, there is no trade-off between computational and sample efficiency: stochastic gradient descent (SGD), with step size equal to $\DStar / \Lip \sqrt{T}$, obtains the optimal dimension-free worst-case bounds for sample, gradient oracle, and computational\footnote{This paper uses the oracle model of computation \cite{nemirovski1983problem}. We also include the gradient oracle cost in the computational complexity. Since for SGD the oracle cost, which involves at minimum reading a vector of size of the dimension, is the dominant term in the upper bounds for computational complexity, the oracle lower bounds imply that these computational complexity upper bounds are optimal.} complexity \cite{agarwal2012information,nemirovski1983problem}.

Parameter-free stochastic convex optimization studies the case where problem parameters such as the distance to optimality are unknown.
Its computational complexity is almost perfectly understood \cite{mcmahan2014unconstrained,cutkosky2018black,attia2024free,carmon2024price,khaled2024tuning,carmon2022making,jacobsen2022parameter}. 
However, the sample complexity of parameter-free stochastic convex optimization remains to be characterized. 
In particular, in the canonical setting where we aim to obtain suboptimality guarantees that hold with constant probability and the Lipschitz constant is known but the distance to optimality is unknown,
existing lower bounds apply to gradient oracle complexity (i.e., the number of gradient evaluations) but not sample complexity \cite[Theorem 2]{carmon2024price}.

A natural approach to addressing the sample complexity of parameter-free optimization is to perform hyperparameter search over different values of the algorithmic parameters to minimize the validation loss. While this is the standard methodology for eliminating hyperparameters, existing literature only provides guarantees for this approach when the loss is bounded \cite{vapnik1999overview}. While some metrics such as accuracy are bounded, other common objectives like cross-entropy loss are not.

\myparagraph{Our contributions} In this paper, we

\begin{enumerate}
\item Show that \emph{standard model selection}, i.e., minimizing the validation error, can catastrophically overfit (e.g., when tuning learning rates) but enjoys strong guarantees when the population risk is strongly convex. %
\item Develop a \emph{generic} reliable model selection method that mitigates \edit{this} risk of overfitting on small validation sets. The method also obtains the same strong guarantees as standard model selection when population risk is strongly convex. \edit{When used for tuning the learning rate, reliable model selection matches the sample complexity of known-parameter stochastic convex optimization up to a $\log\log$ factor in the uncertainty in the distance to optimality.}
\item Develop a regularization-based method that matches the sample complexity of \emph{known}-parameter stochastic optimization when the distance to optimality is unknown\edit{, removing the $\log\log$ factors present in Contribution 2}.
\edit{The key insight is that norm-regularized empirical risk minimization (ERM) can produce an estimate of the distance to optimality to within a constant factor, which is sufficient for optimal known-parameter stochastic optimization methods to achieve optimal sample complexity.}
Due to existing lower bounds \cite{carmon2024price}, our result implies a separation between the gradient oracle complexity and sample complexity of parameter-free stochastic optimization. In contrast, no such separation exists for \emph{known}-parameter stochastic optimization. To obtain these results, we develop new concentration inequalities for the sum of i.i.d. vectors and dependent random variables.
\item Combine our regularization-based method with our reliable model selection method to simultaneously adapt to the Euclidean norm, infinity norm\edit{,} and Manhattan norm for measuring the distance to optimality. 
\item Provide \edit{few}-shot learning and prompt engineering experiments showing that our reliable model selection method can mitigate the risk of overfitting when the validation set is small.
\end{enumerate}

\mysubsection{Related work}{Related Work}

\myparagraph{Guarantees for standard model selection} Under the assumption that the noise in function value evaluations is uniformly bounded, \citet[Theorem 2]{attia2024free} provide guarantees on tuning the learning rate of SGD by performing a grid search using standard model selection. However, this assumption fails even for simple setups like logistic regression. 

\myparagraph{Parameter-free optimization}
We survey related results in terms of 
the Price of Adaptivity (PoA)~\cite{carmon2024price}, a concise framework for describing parameter-free bounds.
The PoA is the ratio between the suboptimality bound obtained when there is uncertainty in the values of the problem parameters and the optimal worst-case suboptimality with known parameters. In other words, it measures how uncertainty in problem parameters degrades worst-case bounds. 

In particular, on $\Lip$-Lipschitz functions with $\Lip$ known but unknown distance to optimality, the PoA is $O(\sqrt{\ln \rho })$ \cite{mcmahan2014unconstrained} where $\rho$ is the multiplicative uncertainty in the initial distance to optimality, i.e., the distance to optimality is in the range $[1,\rho]$. Moreover, the sample complexity of this problem matches this upper bound up to constant factors \cite{carmon2024price}. On the other hand, to guarantee the suboptimality is upper bounded with a constant probability, the corresponding PoA is $O( \ln \ln \rho )$  \cite{carmon2022making}. Moreover, a corresponding lower bound shows that the PoA with respect to stochastic gradient oracle complexity cannot be improved beyond $\Omega( \sqrt{\ln \ln \rho  } )$ \cite{carmon2024price}.
Critically for this paper, the latter lower bound does not preclude the $O(1)$ PoA bound that we show for the sample complexity of parameter-free stochastic optimization (see \Cref{example:combine-methods} and subsequent discussion).

When the Lipschitz constant is also unknown but lies in the range $[1,\ell]$, then the lower bound on the PoA (for constant probability guarantees on the suboptimality) becomes $\Omega( \sqrt{\ln \ln \rho  } + \min\{\ell,\rho\} / \sqrt{T} )$ \cite{carmon2024price}.
This lower bound is matched up to logarithmic factors by a combination of \citet{carmon2022making} and \citet{cutkosky2019artificial}.

There are numerous practical implementations of parameter-free stochastic optimization methods \cite{ivgi2023dog,orabona2017training,chen2022better,kempka2019adaptive,kreisler2024accelerated,defazio2023learning}. 
These papers primarily address large data sets where overfitting is less of a concern. Thus, while our emphasis is on sample efficiency, these papers focus on improving computational efficiency by eliminating the overhead of hyperparameter search.

\myparagraph{Cross-validation}
The holdout method splits the data into a training set and a validation set. It fits the model on the training set and uses the validation set to choose the final model.
A more sophisticated approach is $K$-fold cross-validation, which evenly splits the data into $K$ folds, and for each fold $k$, retrains the model on the other $K-1$ folds and holds out the remaining fold. 
Generalization error is estimated by averaging the error across the $K$ held out folds \cite[Chapter 7]{hastie2009elements}.
Cross-validation requires some form of algorithmic stability to provide theoretical guarantees \cite{kearns1997algorithmic}, but this is not necessary for the holdout method (see \Cref{sec:model-selection-techniques}). 
In practice, cross-validation is often more sample efficient than the holdout method but because it requires $K$ times more compute, it is not popular for compute-constrained machine learning.

\myparagraph{Methods for mitigating overfitting}
\citet{blum2015ladder} develop a method for preventing overfitting to leaderboards for bounded objectives such as accuracy. In contrast, our analysis applies to unbounded losses such as cross-entropy. \citet{breiman1984classification} proposes the one standard error rule for mitigating overfitting in decision trees.
Unfortunately, this rule is restricted to models that can be ordered by a complexity metric and lacks a strong theoretical basis for tuning learning rates.

\subsection{Notation}

Throughout, we assume $f(x;S)$ is convex in $x$ where $S$ is a random variable with sample space $\SS$. 
For conciseness, we denote $f(x;S_i)$ by $f_i(x)$ where $S_i$ represents the $i$th sample.
Define $F(x) := \E[f(x;S)]$, $\Fstar := \inf_{x \in \X}  F(x) > -\infty$, 
$\bar{F}(x)  :=\frac{1}{n} \sum_{i=1}^n f_i(x)$.
Let $\X$ be a closed convex set with $\X \subseteq \R^{d}$. In most practical cases, $\X = \R^d$.
Let
$\XStar := \argmin_{x \in \X} F(x)$ and
$\DStar := \inf_{x \in \XStar} \| x \|$. If $\XStar$ is empty, then $\DStar = \infty$.
The dual norm is
$\| z \|_{*} = \sup \{ z \cdot x : \| x \| \le 1 \}$.
Let  $\grad h(x)$ be any subgradient of a function $h(x)$
and
$\grad_j h(x)$ be the $j$th coordinate of $\grad h(x)$.
Define
$a \wedge b := \min\{ a, b\}$ and $a \vee b := \max\{ a,  b\}$. For a scalar $K$, we let $[K] := \{ 0, 1, \dots, K \}$ and, with slight abuse of notation, for a vector $v$ we 
let $\coor{ v }$ be the $j$th coordinate of $v$. 
Let $\zeros$ be a vector of zeros.
We say the function $h(x)$ is $\Lip$-Lipschitz if and only if for all $x, x' \in \X$, we have
$\abs{h(x) - h(x')} \le \Lip \| x - x' \|$.
An equivalent definition is that $\| \grad h(x) \|_* \le \Lip$ for all $x \in \X$.
We say the function $f$ is $\Lip$-Lipschitz if and only if $h(x)=f(x;S)$ is $\Lip$-Lipschitz almost surely.
We say the function $h : \X \rightarrow \R$ is $\mu$-strongly convex if and only if for all  $u,v \in \X$ we have $h(u) -h(v) \ge \grad h(v) \cdot (u - v) + \frac{\mu}{2} \| u - v \|^2$. We say the function $f$ is $\mu$-strongly convex if and only if $h(x)=f(x;S)$ is $\mu$-strongly convex almost surely.
We assume there exists some reference model $x_0$, a model the user identifies as a reasonable baseline. 
In the theory (\Cref{sec:model-selection-techniques} and \Cref{sec:adaptivity-for-free}),  for simplicity we set the reference
model to be the origin, i.e.,
$x_0 = \zeros$.
Nonetheless, our results hold for an arbitrary reference model by shifting the origin.

\myparagraph{Paper outline} \Cref{sec:model-selection-techniques} analyzes standard model selection and presents our reliable model selection method \edit{which can be used for tuning the learning rate in a grid search}. 
\Cref{sec:adaptivity-for-free} presents our regularization-based method\edit{, which uses norm-regularized ERM to estimate the distance to optimality to within a constant factor, allowing known-parameter stochastic optimization methods to achieve their optimal sample complexity,} and combines it with reliable model selection to simultaneously adapt to multiple problem structures.
\Cref{sec:experiments} provides experiments indicating that our reliable model selection method can mitigate the risk of overfitting.

\mysection{Model selection techniques}{Model Selection Techniques}\label{sec:model-selection-techniques}

In this section, we consider the case where there is a finite set of models $x_0, x_1, \dots, x_K$.
In this setting, our goal is to find the best model
among these $K \edit{+ 1}$ models.
We will also explore applications of these
techniques to selecting hyperparameters 
for stochastic optimization methods.

\Cref{sec:standard-model-selection} studies the performance of standard model selection which chooses the model that minimizes the validation error.
In particular, we show that it can effectively tune hyperparameters when strong convexity is present but otherwise may perform poorly compared to parameter-free optimization methods.
\Cref{sec:reliable-model-selection} studies the performance of a new algorithm called \RMS{}, which obtains good worst-case performance both with and without strong convexity.

\mysubsection{When does standard model selection succeed and fail?}{When Does Standard Model Selection Succeed and Fail?}\label{sec:standard-model-selection}

Recall that standard model selection picks the model with the smallest validation error, i.e., $\kGreedy \in \argmin_{k \in [K]} \barF{k}$ where $\barF{k} = \frac{1}{n} \sum_{i=1}^n f_i(x_k)$. It is well-known that for bounded losses, i.e., $f(x; S) \in [0,1]$ almost surely for all $x \in \X$, one can show that, using Hoeffding's inequality and a union bound, with probability at least $1-\delta$ for any $\delta \in (0,1)$, 
$F(x_{\kGreedy}) \le \min_{k \in [K]} F(x_k) + \sqrt{\edit{2} \ln (2 \edit{(K+1)}/\delta) / n}$.

\Cref{prop:greedy-strong-convexity} shows that standard model selection can more effectively optimize hyperparameters when strong convexity is present.
For example, if we apply \Cref{prop:greedy-strong-convexity} to 
\cite[Theorem 3]{hazan2014beyond} with Lipschitz constant $\Lip$ known and set $x_k$ to be the output of their algorithm with strong convexity parameter $\mu_k = \mu_0 \e^{k}$ and if $\mu \in [\mu_0, \mu_K]$ we obtain a suboptimality of at most 
$O\big( \frac{\Lip^2}{\mu n} \big( \ln \frac{\ln(\mu_K / \mu_0)}{\delta} + \ln \ln n \big) \big)$. \citet[Theorem \edit{4}]{carmon2022making} obtain a similar
result with slightly worse sample complexity but \edit{a} computational complexity that is a log factor \edit{better}.
Adding $\Lip$ to the grid search yields similar guarantees when the Lipschitz constant is also unknown.
\ShortPaper{The proof of \Cref{prop:greedy-strong-convexity} appears in \Cref{app:proof-lem:greedy-strong-convexity}.
It}\LongPaper{The proof of \Cref{prop:greedy-strong-convexity}} uses strong convexity to argue that solutions, $x_k$, with better values of the population risk, $F(x)$, are closer to the optimal solution, which due to the assumption that $f$ is Lipschitz ensures $\abs{f(x_k; S)-f(\xstar; S)}$ is small, and this enables higher quality model selection.

\begin{proposition}\label{prop:greedy-strong-convexity}
Suppose that $f$ is $\Lip$-Lipschitz and $F$ is $\mu$-strongly convex. Then, for all $\delta \in (0,1)$ with probability at least $1-\delta$ we have
\[
F(x_{\kGreedy}) - \Fstar \le 2 \max\left\{ \min_{k \in [K]} F(x_{k}) - \Fstar, \frac{32 \Lip^2 }{\mu n} \ln \frac{2 \edit{(K+1)}}{\delta} \right\}.
\]
\end{proposition}

\newcommand{\GreedyStrongConvexityProof}[0]{
\begin{proof}
Let 
$\varepsilon_k := F(x_k) - \Fstar$,
and $\kStar \in \argmin_{k \in [K]} F(x_k)$.
With probability $1-\delta$, 
for all $k \in [K]$,
\begin{flalign}
\abs{\varepsilon_{k} - (\barF{k} - \bar{F}(\xstar))} &\us{\le}{i} 2 \Lip \| x_k - \xstar \| \sqrt{\frac{\ln \frac{2 \edit{(K+1)}}{\delta}}{2 n}} \us{\le}{ii} 2 \Lip \sqrt{\frac{\varepsilon_{k} \ln \frac{2 \edit{(K+1)}}{\delta}}{\mu n}} \label{eq:bound-varepsilon-k}
\end{flalign}
where $(i)$ uses a union bound and Hoeffding's inequality (\Cref{thm:hoeffdings-inequality} in \Cref{app:well-known-concentration-inequalities}) with $Z_i = f_i(x_k) - f_i(\xstar)$ and $\abs{Z_i} \le \Lip \| x_k - \xstar \|$ since $f$ is $\Lip$-Lipschitz, and $(ii)$ uses strong convexity (in particular,   \Cref{eq:useful-strong-convexity-implication} in \Cref{app:well-known-concentration-inequalities}).
Substituting $k = \kGreedy$ into \Cref{eq:bound-varepsilon-k} gives
\begin{flalign}\label{eq:greedy-ub} 
\varepsilon_{\kGreedy} \le \barF{\kGreedy} - \bar{F}(\xstar) + 2 \Lip \sqrt{\frac{\varepsilon_{\kGreedy} \ln \frac{2 \edit{(K+1)}}{\delta}}{\mu n }}.
\end{flalign}
Similarly, for $k = \kStar$ we have 
\begin{flalign}\label{eq:k-star-error} 
\barF{\kStar} - \bar{F}(\xstar) \leq \varepsilon_{\kStar} + 2 \Lip \sqrt{\frac{\varepsilon_{\kStar} \ln \frac{2 \edit{(K+1)}}{\delta}}{\mu n}}.
\end{flalign}
Next,
\begin{flalign*}
\varepsilon_{\kGreedy} - \varepsilon_{\kStar} &\us{\le}{i}  2 \Lip \sqrt{ \frac{\ln \frac{2 \edit{(K+1)}}{\delta} }{\mu n}} \left(  \sqrt{\varepsilon_{\kStar}}  + \sqrt{\varepsilon_{\kGreedy}} \right) \us{\le}{ii} \sqrt{ \frac{32 \Lip^2 \ln \frac{2 \edit{(K+1)}}{\delta} }{\mu n} \cdot \frac{\varepsilon_{\kGreedy}}{2}} \\
&\le \max\left\{\frac{32 \Lip^2 \ln \frac{2 \edit{(K+1)}}{\delta} }{\mu n}  , \frac{\varepsilon_{\kGreedy}}{2}  \right\}
\end{flalign*}
where $(i)$ uses \Cref{eq:greedy-ub}, \Cref{eq:k-star-error} and $\barF{\kGreedy}  \le \barF{\kStar}$, and $(ii)$ uses that $\varepsilon_{\kStar} \le \varepsilon_{\kGreedy}$.
Analyzing this inequality in the case that $\frac{32 \Lip^2 \ln \frac{2 \edit{(K+1)}}{\delta} }{\mu n} \le \frac{\varepsilon_{\kGreedy}}{2}$ and $\frac{32 \Lip^2 \ln \frac{2 \edit{(K+1)}}{\delta} }{\mu n} > \frac{\varepsilon_{\kGreedy}}{2}$ gives the desired result.
\end{proof}
}
\LongPaper{\GreedyStrongConvexityProof{}}

Conversely, without strong convexity, standard model selection can perform poorly at tuning learning rates. 
For a concrete example, consider adaptive SGD (also known as isotropic AdaGrad \cite{gupta2017unified}), which is a scalar variant of \textsc{AdaGrad} \cite{duchi2011adaptive}.
While we chose adaptive SGD as an example, similar issues occur for tuning other standard stochastic optimization methods \edit{such as} regularized ERM or SGD with a fixed step size.
The formula for adaptive SGD is
\begin{flalign} \label{eq:adaptive-SGD}
	u_{t+1} \gets  \proj_{R}\Bigg( u_t - R \frac{\grad f_{n+t}(u_t)}{\sqrt{\sum_{j=1}^t \| \grad f_{n+j}(u_j) \|_2^2}} \Bigg) , ~
\bar{u}_t \gets \frac{1}{t} \sum_{j=1}^{t} u_j,	~  \edit{\text{for } t=1, \dots, n}
\end{flalign}
where \edit{$u_1 = \zeros$}, $\proj_{R}(\hat{u}) := \argmin_{u \in \X : \| u \|_2 \le R} \| u - \hat{u} \|_2$ and $R > 0$ is the learning rate which represents our estimated distance to optimality. We add the constraint $\| u \|_2 \le R$ to ensure the domain is bounded proportionally to the learning rate.
To clearly separate the validation set and training set, \Cref{eq:adaptive-SGD} uses gradient $\grad f_{n+t}(u_t)$ instead of $\grad f_{t}(u_t)$.
Thus, we can think of $f_1, \dots, f_n$ as the validation set and $f_{n+1}, \dots, f_{2n}$ as the training set. 
Denote $\AdaSGD{R} = \bar{u}_n$.
It is well-known, e.g., \citet[Theorem 3.9 \& 4.14]{orabona2021modern}, that assuming $f$ is $\Lip$-Lipschitz, adaptive SGD starting from $u_1 = \zeros$, with probability $1-\delta$ obtains 
\begin{flalign}\label{thm:adaptive-SGD}
F( \AdaSGD{R}  )  - \min_{x \in \X : \| x \|_2 \le R } F(x) \le O\bigg(  \frac{\Lip R}{\sqrt{n}} \sqrt{\ln 2/\delta } \bigg)
\end{flalign}
Setting $R = \DStar$ yields the optimal $O(\Lip \DStar \sqrt{\ln 2/\delta }  / \sqrt{n})$ suboptimality guarantee. However, this approach relies on $\DStar$ being known.

The standard approach to selecting the learning rate is to run the method across a grid of learning rates and then choose the solution that minimizes the validation loss.
Unfortunately, this can lead to poor performance in some situations, as illustrated by \Cref{prop:greedy-model-selection-performs-poorly}. We emphasize that this is not a critique of adaptive SGD, but an example of a broader limitation of standard model selection for hyperparameter tuning in the absence of strong convexity.

\begin{proposition}\label{prop:greedy-model-selection-performs-poorly}
Assume the validation and training set each consist of $n \ge 3000$ i.i.d. samples. Let $x_k = \AdaSGD{\eta_k}$ where $\eta_0, \eta_1, \dots, \eta_K$ are the learning rates we will evaluate.
Then, there exists a $1$-Lipschitz stochastic convex optimization problem with minimizer at the origin
such that with probability at least $1/1000$, $F(x_{\kGreedy}) - \Fstar  \ge \frac{1}{288 \sqrt{n}} \max_{k \in [K]} \eta_k$.
\end{proposition}

The proof of \Cref{prop:greedy-model-selection-performs-poorly} appears in \Cref{app:standard-model-selection-slow}.
The lower bound is based on the  stochastic optimization problem given by  $f(x; 0) = \abs{x}$, $f(x; 1) = -x$,
and $S$ is a Bernoulli random variable with success probability $q = 1/2 - n^{-0.5}/16$. \edit{For this problem, the population objective takes the form
\[
F(x) = \begin{cases}
-x & x \le 0 \\
\frac{x}{8 \sqrt{n}} & x > 0
\end{cases}
\]
whereas the validation and training objectives take the form
\[
\frac{1}{n} \sum_{i=1}^n f_i(x)= \begin{cases}
	-x & x \le 0 \\
	\frac{n - 2 \sum_{i=1}^n S_i}{n} x & x > 0,
\end{cases}  \quad \quad \frac{1}{n} \sum_{i=n+1}^{2n} f_i(x) = \begin{cases}
	-x & x \le 0 \\
	\frac{n - 2 \sum_{i=n+1}^{2 n} S_i}{n} x & x > 0,
\end{cases}
\]
respectively.
One can show with} constant probability that
$\sum_{i=n+1}^{2n} S_i \ge n/2 + \Omega(\sqrt{n})$ \edit{making the training objective slope negative for $x > 0$}. 
In this case, adaptive gradient descent 
with learning rate $\max_{k \in [K]} \eta_k$ terminates with $\bar{u}_n \ge \max_{k \in [K]} \eta_k/36$.
Moreover, with constant probability, 
the validation set will contain $f(x; 1)$ more often than $f(x; 0)$, i.e., $\sum_{i=1}^n S_i > n/2$\edit{, making the validation objective slope negative for $x > 0$}. Consequently $x_{\kGreedy} \ge \max_{k \in [K]} \eta_k/36$ \edit{and thus $F(x_{\kGreedy}) - \Fstar = F(x_{\kGreedy})  \ge \frac{1}{288 \sqrt{n}} \max_{k \in [K]} \eta_k$}.

\edit{We emphasize that this lower-bound construction exploits two features of the example: the absence of strong convexity, which permits instability in the training and validation objectives, and the unboundedness of the losses, which allows the suboptimality to grow as the iterates diverge.}

In comparison, parameter-free methods \cite{carmon2022making} in the setting of \Cref{prop:greedy-model-selection-performs-poorly} (i.e., there is a known Lipschitz constant of $1$), can obtain 
a suboptimality of $O( \ln(\ln ( n \DStar / \eta_{\min}))  (\DStar + \eta_{\min}) / \sqrt{n} )$ with constant probability, where $\eta_{\min}$ is a \emph{known} lower bound on $\DStar$. 
This bound is much less sensitive to uncertainty in the distance to optimality
compared with standard model selection. For example, if the distance to optimality is between $1$ and $\rho$ then \Cref{prop:greedy-model-selection-performs-poorly} shows standard model selection will be a factor of $\Omega(\rho)$ worse than if we knew the distance to optimality versus a factor of $O(\ln \ln (n \rho))$ for \cite{carmon2022making}.

\mysubsection{Our Reliable Model Selection method}{Our Reliable Model Selection Method}\label{sec:reliable-model-selection}

Next, we consider how to reliably select high-quality models and avoid the issues that plagued standard model selection for tuning learning rates in \Cref{sec:standard-model-selection}. Our method, described in \Cref{alg:reliable-hyperparameter-selection}, requires a list of confidence interval widths $\tau_1, \dots, \tau_{K}$ that satisfy \Cref{assume:confidence-intervals}.
The algorithm identifies a set of safe models $\mathcal{F}$ over which we minimize the validation error. The algorithmic parameter $\gamma \in [1,\infty)$ controls the size of $\mathcal{F}$, with larger values of $\gamma$ behaving more like standard model selection. 
When $\gamma=1$, the algorithm reduces to minimizing the upper confidence bound on the function value, but this can be overly  conservative, heavily favoring models that make predictions similar to the reference model.

\begin{condition}\label{assume:confidence-intervals}
	Let $\tau_1, \dots, \tau_{K}$ be nonnegative scalars such that for some $\delta \in (0,1)$,
$\P( \exists k \in \edit{\{ 1, \dots, K \}} : \abs{F(x_k) - \barF{k}  - (F(\zeros)  - \bar{F}(\zeros))} > \tau_k) \le \delta$.
\end{condition}

\Cref{assume:confidence-intervals} supposes that we have confidence intervals on the closeness of our validation loss to the population loss relative to the reference model, $x_0 = \zeros$.
Depending on the problem at hand, there are multiple ways to generate such confidence intervals.
For example, if $f$ is $\Lip$-Lipschitz and some upper bound $\LipHat \ge \Lip$ is known, then a result of   \citet{maurer2009empirical} (\Cref{thm:empirical-bennett} in \Cref{app:well-known-concentration-inequalities}) with $Z_i = f_i(x_k) -f_i(\zeros)$ gives that with probability at most $\delta_k \in (0,1)$,
$\abs{F(x_k) - F(\zeros)  - (\barF{k} - \bar{F}(\zeros))} > \sqrt{\frac{2 \SampleVar{k}}{n} \ln \frac{4}{\delta_k} } + \frac{14 \LipHat \| x_k \|}{3 (n-1)} \ln \frac{4}{\delta_k}$
where
\[
\SampleVar{k} := \frac{1}{n-1} \sum_{i=1}^n (f_i(x_k) - f_i(\zeros) - (\barF{k}  - \bar{F}(\zeros)))^2
\]
is the sample variance \edit{of the paired differences $f_i(x_k) - f_i(\zeros)$}.
Setting $\delta_k = \frac{\delta}{K}$  and applying a union bound shows \Cref{assume:confidence-intervals} holds with
\begin{flalign}
	\tau_k &= \sqrt{\frac{2 \tilde{c} \SampleVar{k}}{n}} + \tilde{c} \frac{14 \LipHat \| x_k \|}{3 (n-1)} \text{ where } \tilde{c} := \ln \frac{4 K}{\delta}. \label{define:theta-k}
\end{flalign}

\begin{algorithm}[!tbh]
	\begin{minipage}{0.48\textwidth}
\begin{algorithmic}[1]
		\INPUT A scalar $\gamma \in [1,\infty)$
		\INPUT Candidate solutions $x_0, \dots, x_K$
		\INPUT Sample functions 
		$f_1, \dots, f_n$ 
		\INPUT $\tau_1, \dots, \tau_K$ satisfying \Cref{assume:confidence-intervals}. 
		\STATE $\tau_0 := 0$ \vspace{0.09cm}
		\STATE $\theta \gets \min_{k \in [K]} \barF{k} + \gamma \tau_k$ \vspace{0.09cm}
		\STATE $\mathcal{F} \gets \{ k \in [K] : \barF{k} + \tau_k \le \theta \} $ \vspace{0.09cm}
		\STATE $\kReliable := \argmin_{k \in \mathcal{F}}  \barF{k}$
		\vspace{0.09cm}
		\OUTPUT The chosen index: $\kReliable$.
\end{algorithmic}
 \end{minipage}
\arxiv{\hfill}
\begin{minipage}{0.45\textwidth}
	\centering
	\begin{tikzpicture}
		\begin{axis}[
			width=8cm,
			height=4.5cm,
			xlabel={Models},
			ylabel={Validation Error},
			xtick={0,1,2,3,4},
			xticklabels={$x_0$, $x_1$, $x_2$, $x_3$, $x_4$},
			xlabel style={yshift=-6pt},
			ytick=\empty,
			ymin=0.24, ymax=0.65,
			axis lines=left,
			x axis line style={-},
			grid=none,
			clip=false
		]
		
		\addplot+[
			only marks,
			mark=*,
			color=black,
			mark options={fill=black, thick}
		] coordinates {
			(0, 0.43)
			(1, 0.37)
			(2, 0.35)
			(3, 0.5)
			(4, 0.33)
		};
		
		\addplot+[
			mark=none,
			draw=none,
			error bars/.cd,
			y dir=plus,
			y explicit,
			error bar style={color=red, line width=1.25pt},
			error mark options={rotate=90, mark size=8pt, line width=1.25pt}
		] coordinates {
			(0, 0.43) +- (0, 0.00)
			(1, 0.37) +- (0, 0.04)
			(2, 0.35) +- (0, 0.12)
			(3, 0.5)  +- (0, 0.15)
			(4, 0.33) +- (0, 0.30)
		};
		
		\addplot+[
			mark=none,
			draw=none,
			error bars/.cd,
			y dir=plus,
			y explicit,
			error bar style={color={rgb,255:red,255;green,127;blue,14}, line width=1.25pt},
			error mark options={rotate=90, mark size=8pt, line width=1.25pt}
		] coordinates {
			(0, 0.43) +- (0, 0.00)
			(1, 0.37) +- (0, 0.0133)
			(2, 0.35) +- (0, 0.04)
			(3, 0.5)  +- (0, 0.05)
			(4, 0.33) +- (0, 0.10)
		};
		
		\addplot[red, dashed, thick] coordinates {(-0.5,0.41) (4.5,0.41)};
		\node[anchor=west, red] at (axis cs:4.5,0.41) {\small\shortstack{$\theta$}};
		\node[anchor=north, color={rgb,255:red,31;green,119;blue,180}] at (axis cs:4,0.33) {\small\shortstack{$\kGreedy$}};
		\node[anchor=north, color={rgb,255:red,255;green,127;blue,14}] at (axis cs:2,0.35) {\small\shortstack{$\kReliable$}};
		\node at (axis cs:1.5, 0.15) {\small$\underbrace{\hspace{2.25cm}}_{\mathcal{F}}$};
		
		\end{axis}
		\end{tikzpicture}
\end{minipage}
\caption{Reliable model selection method with illustration \\
	{\small Black dots show the validation errors of candidate models $x_0$ to $x_4$. Orange and red intervals depict $\tau_k$ and $\gamma \tau_k$, respectively, which are smaller for models with predictions more similar to the reference model, $x_0$. 
The threshold $\theta$ (red dashed line) is the smallest extended upper bound $\barF{k} + \gamma \tau_k$. Models with $\barF{k} + \tau_k \le \theta$ form the candidate set $\mathcal{F}$ (in this example $\mathcal{F} = \{1, 2\}$). Among these, $x_2$ has the lowest error and thus $\kReliable = 2$. 
Standard model selection chooses the lowest validation error across all models, $\kGreedy = 4$.}
}
\label{alg:reliable-hyperparameter-selection}
\end{algorithm}

The following lemma is key to establishing the main guarantees of our algorithm (\Cref{thm:reliable-guarantees}).

\begin{lemma}\label{lem:reliable-selection-general-guarantee} With probability at least  $1-\delta$, we have $F(x_{\kReliable}) \le F(x_{k})  + (1 + \gamma) \tau_k$ for all $k \in [K]$.
\end{lemma}
\newcommand{\overle}[1]{\overset{(#1)}{\le}}

\begin{proof}
	For all $k \in [K]$ we have 
	\begin{flalign*}
	0 &\overle{i} \theta -  \barF{\kReliable}  - \tau_{\kReliable} \overle{ii} \barF{k} + \gamma \tau_k - \barF{\kReliable}  - \tau_{\kReliable} \\
	&=  \gamma \tau_k + \barF{k}  - \bar{F}(\zeros) + \bar{F}(\zeros) - \barF{\kReliable}  - \tau_{\kReliable} \overle{iii}  (1 + \gamma) \tau_k + F(x_k) - F(x_{\kReliable})
	\end{flalign*}
	due to $(i)$  $\kReliable \in \mathcal{F}$, $(ii)$ the definition of $\theta$, and $(iii)$ 
	\Cref{assume:confidence-intervals}.
\end{proof}

\begin{theorem}\label{thm:reliable-guarantees}
Suppose that $f$ is $\Lip$-Lipschitz and $\tau_k$ is defined as per \Cref{define:theta-k}.
Let $\delta \in (0,1)$ and $\tilde{c} := \ln \frac{4 K}{\delta}$.
Then, with probability at least $1-\delta$,
\begin{flalign}\label{eq:F-rely-x-k-guarantee}
F(x_{\kReliable}) \le \min_{k \in [K]} F(x_{k}) + (1 + \gamma) \left(\Lip  \sqrt{\frac{2 \tilde{c}}{n\edit{-1}}} + \frac{14 \LipHat \tilde{c}}{3(n\edit{-1})} \right) \| x_k \|.
\end{flalign}
Additionally, if $F$ is $\mu$-strongly convex and $\gamma \ge 3$, then, with probability at least $1-\edit{2}\delta$,
\begin{flalign}\label{eq:F-rely-strong-convexity-guarantee}
F(x_{\kReliable}) - \Fstar  \le 4 \Big( \min_{k \in [K]}   F(x_{k}) - \Fstar \Big) +  \frac{128}{\mu} \bigg( 2 \Lip \sqrt{\frac{2 \tilde{c}}{n-1}} + \frac{14 \LipHat \tilde{c}}{3 \edit{(n-1)}}  \bigg)^2.
\end{flalign}
\end{theorem}

\newcommand{\ProofOfReliableProp}[0]{
\begin{proof}
Applying that $f$ is $L$-Lipschitz to \Cref{define:theta-k} gives 
\[
\tau_k \le \Lip \| x_k \| \sqrt{\frac{2 \tilde{c}}{n\edit{-1}}} + \frac{14 \LipHat \tilde{c} \| x_k \|}{3(n\edit{-1})}.
\]
Applying \Cref{lem:reliable-selection-general-guarantee} using this upper bound on $\tau_k$ along with a union-bound over the candidates implies \Cref{eq:F-rely-x-k-guarantee}.
It remains to show \Cref{eq:F-rely-strong-convexity-guarantee}.
	
Let 
$\varepsilon_k := F(x_k) - \Fstar$,
and $\kStar \in \argmin_{k \in [K]} F(x_k)$.
Observe from the definition of $\SampleVar{k}$ that
\begin{flalign*} 
&\sqrt{\SampleVar{\kStar}} - \sqrt{\SampleVar{k}} \\
 &= \sqrt{\sum_{i=1}^n \frac{(f_i(x_{\kStar})- \barF{\kStar} - (f_i(x_0) -\barF{0}) )^2}{n-1}} - \sqrt{\sum_{i=1}^n \frac{(f_i(x_k)- \barF{k} - (f_i(x_0) -\barF{0}) )^2}{n-1}} \\
&\us{\le}{i} \sqrt{\sum_{i=1}^n \frac{(f_i(x_{\kStar})- \barF{\kStar} - (f_i(x_k) -\barF{k}) )^2}{n-1}} = \sqrt{\sum_{i=1}^n \frac{(f_i(x_{\kStar})- f_i(x_k)  - (\barF{\kStar} -\barF{k}) )^2}{n-1}} \\
&\us{\le}{ii} 2 \Lip \| x_{\kStar} - x_k \|\sqrt{\frac{n}{n-1}}
\end{flalign*} 
where $(i)$ uses the triangle inequality and $(ii)$ uses that $f$ is $L$-Lipschitz.
Thus, by \Cref{define:theta-k},
\begin{flalign} 
\tau_{\kStar} - \tau_{k} &= \sqrt{\frac{2 \tilde{c}}{n}} \left( \sqrt{\SampleVar{\kStar}} - \sqrt{\SampleVar{k}} \right) +  \frac{14 \LipHat \tilde{c}}{3 \edit{(n-1)}} (\| x_{\kStar} \| -  \| x_{k} \| ) \notag \\
&\le 2 \Lip \sqrt{\frac{2 \tilde{c}}{n-1}} \| x_{\kStar} - x_k \| +  \frac{14 \LipHat \tilde{c}}{3 \edit{(n-1)}} (\| x_{\kStar} \| -  \| x_{k} \| ) \notag \\
&\us{\le}{i} \left( 2 \Lip \sqrt{\frac{2 \tilde{c}}{n-1}} +   \frac{14 \LipHat \tilde{c}}{3 \edit{(n-1)}}  \right)  \| x_{\kStar} - x_k \|
\label{eq:t-star-minus-tau-k} 
\end{flalign} 
where $(i)$ uses the triangle inequality.

Recall that \Cref{assume:confidence-intervals} states that, with probability $1-\delta$,
\begin{flalign}\label{eq:cond-tau-k-case}
\abs{F(x_k) - \barF{k}  - (F(\zeros)  - \bar{F}(\zeros))} \le \tau_k \quad \forall k \in [K].
\end{flalign}
The rest of the proof will proceed in the case that \Cref{eq:cond-tau-k-case} holds.

Let $k' \in \argmin_{k \in [K]} \barF{k} + \gamma \tau_k$.
Consider the case that $\barF{\kStar}+ \tau_{\kStar} > \barF{k'} + \gamma \tau_{k'}$. Then,
by \Cref{eq:cond-tau-k-case} with $k=\kStar$ and $k=k'$,
\begin{flalign}
	0 &< \barF{\kStar}+ \tau_{\kStar} - \barF{k'} -  \gamma \tau_{k'} = \barF{\kStar} - \bar{F}(\zeros)  + \tau_{\kStar} + \bar{F}(\zeros) - \barF{k'} -  \gamma \tau_{k'} \notag \\
	&\le F(x_{\kStar})  + 2 \tau_{\kStar} - F(x_{k'}) -  (\gamma - 1) \tau_{k'} = \varepsilon_{\kStar} - \varepsilon_{k'} + 2 \tau_{\kStar} -  (\gamma - 1) \tau_{k'}. \label{eq:eps-gap-tau-gamma-bound}
\end{flalign}
Thus, 
\begin{flalign*}
	\varepsilon_{k'} - \varepsilon_{\kStar} &\us{\le}{i} 2 \tau_{\kStar} - (\gamma - 1) \tau_{k'} \us{\le}{ii} 2 \left( 2 \Lip \sqrt{\frac{2 \tilde{c}}{n-1}} +  \tilde{c} \frac{14 \LipHat}{3 \edit{(n-1)}}  \right)  \| x_{\kStar} - x_{k'} \| - (\gamma - 3) \tau_{k'} \\
	&\us{\le}{iii}   2 \left( 2 \Lip \sqrt{\frac{2 \tilde{c}}{n-1}} +  \tilde{c} \frac{14 \LipHat}{3 \edit{(n-1)}}  \right) \left( \sqrt{\frac{2 \varepsilon_{\kStar}}{\mu}} + \sqrt{\frac{2 \varepsilon_{k'}}{\mu}} \right) \\
	&\us{\le}{iv} 2 \left( 2 \Lip \sqrt{\frac{2 \tilde{c}}{n-1}} +  \tilde{c} \frac{14 \LipHat}{3 \edit{(n-1)}}  \right) 2 \sqrt{\frac{2 \varepsilon_{k'}}{\mu}} = \frac{8}{\sqrt{\mu}} \left( 2 \Lip \sqrt{\frac{2 \tilde{c}}{n-1}} +  \tilde{c} \frac{14 \LipHat}{3 \edit{(n-1)}}  \right) \sqrt{\frac{ \varepsilon_{k'}}{2}} \\
	&\us{\le}{v} \max\left\{ \frac{64}{\mu} \left( 2 \Lip \sqrt{\frac{2 \tilde{c}}{n-1}} +  \tilde{c} \frac{14 \LipHat}{3 \edit{(n-1)}}  \right)^2, \frac{\varepsilon_{k'}}{2} \right\}
\end{flalign*}
where $(i)$ uses \Cref{eq:eps-gap-tau-gamma-bound}, $(ii)$ uses \Cref{eq:t-star-minus-tau-k},
$(iii)$ uses the triangle inequality, \Cref{eq:useful-strong-convexity-implication}, and $\gamma \in [3, \infty)$, $(iv)$ uses $\varepsilon_{\kStar} \le \varepsilon_{k'}$, and $(v)$ uses that $a b \le \max\{a^2, b^2\}$. Analyzing the latter inequality in the cases where $\frac{\varepsilon_{k'}}{2} > \frac{64}{\mu} \left( 2 \Lip \sqrt{\frac{2 \tilde{c}}{n-1}} +  \tilde{c} \frac{14 \LipHat}{3 \edit{(n-1)}}  \right)^2$ and $\frac{\varepsilon_{k'}}{2} \le \frac{64}{\mu} \left( 2 \Lip \sqrt{\frac{2 \tilde{c}}{n-1}} +  \tilde{c} \frac{14 \LipHat}{3 \edit{(n-1)}}  \right)^2$  yields:
\begin{flalign}\label{eq:subopt-guarantee-for-k-prime}
\varepsilon_{k'} \le \max\left\{ 2 \varepsilon_{\kStar}, \varepsilon_{\kStar} + \frac{64}{\mu} \left( 2 \Lip \sqrt{\frac{2 \tilde{c}}{n-1}} +  \tilde{c} \frac{14 \LipHat}{3 \edit{(n-1)}}  \right)^2 \right\}.
\end{flalign}
On the other hand, if $\barF{\kStar}+ \tau_{\kStar} \le \barF{k'} + \gamma \tau_{k'}$, i.e., $\kStar \in \mathcal{F}$,
then \Cref{eq:subopt-guarantee-for-k-prime} also holds. Thus, it remains to bound $\varepsilon_{\kReliable}$ under \Cref{eq:subopt-guarantee-for-k-prime}. 
The proof proceeds exactly as per the proof of \Cref{prop:greedy-strong-convexity} (with $k'$ replacing $\kStar$ and $\kReliable$ replacing $\kGreedy$) but we include it for completeness. 
With probability $1-\delta$, 
for all $k \in [K]$,
\begin{flalign}
	\abs{\varepsilon_{k} - (\barF{k} - \bar{F}(\xstar))} &\us{\le}{i} 2 \Lip \| x_k - \xstar \| \sqrt{\frac{\ln \frac{2 (K\edit{+1})}{\delta}}{2 n}} \us{\le}{ii} 2 \Lip \sqrt{\frac{\varepsilon_{k} \ln \frac{\edit{4} K}{\delta}}{\mu n}} \label{eq:bound-varepsilon-k-reliable}
\end{flalign}
where $(i)$ uses a union bound and Hoeffding's inequality (\Cref{thm:hoeffdings-inequality} in \Cref{app:well-known-concentration-inequalities}) with $Z_i = f_i(x_k) - f_i(\xstar)$ and $\abs{Z_i} \le \Lip \| x_k - \xstar \|$ by the assumption that $f$ is $\Lip$-Lipschitz, and $(ii)$ uses \Cref{eq:useful-strong-convexity-implication}.
Substituting $k = \kReliable$ into \Cref{eq:bound-varepsilon-k-reliable} gives
\begin{flalign}\label{eq:greedy-ub-reliable} 
	\varepsilon_{\kReliable} \le \barF{\kReliable} - \bar{F}(\xstar) + 2 \Lip \sqrt{\frac{\varepsilon_{\kReliable} \ln \frac{\edit{4} K}{\delta}}{\mu n }}.
\end{flalign}
Similarly, for $k = k'$ we have 
\begin{flalign}\label{eq:k-prime-error-reliable} 
	\barF{k'} - \bar{F}(\xstar) \leq \varepsilon_{k'} + 2 \Lip \sqrt{\frac{\varepsilon_{k'} \ln \frac{\edit{4}K}{\delta}}{\mu n}}.
\end{flalign}
Next,
\begin{flalign*}
	\varepsilon_{\kReliable} - \varepsilon_{k'} &\us{\le}{i}  2 \Lip \sqrt{ \frac{\ln \frac{\edit{4} K}{\delta} }{\mu n}} \left(  \sqrt{\varepsilon_{k'}}  + \sqrt{\varepsilon_{\kReliable}} \right) \us{\le}{ii} \sqrt{ \frac{32 \Lip^2 \ln \frac{\edit{4} K}{\delta} }{\mu n} \cdot \frac{\varepsilon_{\kReliable}}{2}} \\ 
	&\le \max\left\{\frac{32 \Lip^2 \ln \frac{\edit{4} K}{\delta} }{\mu n}  , \frac{\varepsilon_{\kReliable}}{2}  \right\}
\end{flalign*}
where $(i)$ uses \Cref{eq:greedy-ub-reliable}, \Cref{eq:k-prime-error-reliable} and $\barF{\kReliable}  \le \barF{k'}$, and $(ii)$ uses that $\varepsilon_{k'} \le \varepsilon_{\kReliable}$.
Analyzing this inequality in the case that $\frac{32 \Lip^2 \ln \frac{\edit{4} K}{\delta} }{\mu n} \le \frac{\varepsilon_{\kReliable}}{2}$ and $\frac{32 \Lip^2 \ln \frac{\edit{4} K}{\delta} }{\mu n} > \frac{\varepsilon_{\kReliable}}{2}$ gives
\begin{flalign}\label{eq:eps-k-reliable-bound-in-terms-of-k'}
\varepsilon_{\kReliable} \le \max\left\{ 2 \varepsilon_{k'}, \varepsilon_{k'} + \frac{32 \Lip^2 \ln \frac{\edit{4} K}{\delta} }{\mu n} \right\}.
\end{flalign}
Combining \Cref{eq:subopt-guarantee-for-k-prime} and \Cref{eq:eps-k-reliable-bound-in-terms-of-k'}\edit{, 
and taking a union bound over \Cref{eq:cond-tau-k-case} and \Cref{eq:bound-varepsilon-k-reliable},} gives \Cref{eq:F-rely-strong-convexity-guarantee}.
\end{proof}
}
\LongPaper{\ProofOfReliableProp}

\Cref{thm:reliable-guarantees} is the main guarantee for our \RMS{} method. \ShortPaper{The proof of \Cref{thm:reliable-guarantees} appears in \Cref{app:prove:prop:reliable-strongly-convex}.}
\Cref{eq:F-rely-x-k-guarantee} follows immediately from \Cref{define:theta-k} and \Cref{lem:reliable-selection-general-guarantee}.
To show \Cref{eq:F-rely-strong-convexity-guarantee} we  establish that there exists a $k \in \mathcal{F}$ with a good objective value, and then we follow the proof of \Cref{prop:greedy-strong-convexity} to show that $\kReliable$ is almost as good as the best model in $\mathcal{F}$.

We can use \Cref{eq:F-rely-x-k-guarantee} to generically 
obtain sample complexity bounds for parameter-free stochastic optimization. \edit{
For example, \Cref{cor:rms-parameter-free-rate} shows how to obtain such bounds by applying \RMS{} to tune the learning rate for adaptive SGD.
To interpret \Cref{cor:rms-parameter-free-rate},
recall that  $f_1, \dots, f_{2n}$ are  i.i.d. samples where 
$f_1, \dots, f_n$ is the validation set (i.e., used by \RMS) and  $f_{n+1}, \dots f_{2n}$ is the training set, i.e., used by adaptive SGD, see \Cref{eq:adaptive-SGD}.}

\edit{
\begin{corollary}\label{cor:rms-parameter-free-rate}
Suppose $f$ is $\Lip$-Lipschitz with known upper bound $\LipHat \ge \Lip$, and that $\DStar \in [1, \rho]$ for some known $\rho > 1$. For $k = 1, 2, \dots, \lceil \ln \rho \rceil$, let $x_k = \AdaSGD{\e^k}$ be the output of adaptive SGD, i.e., \Cref{eq:adaptive-SGD}, with learning rate $R_k = \e^k$. Then, for any $\delta \in (0,1)$, \Cref{alg:reliable-hyperparameter-selection}  applied to $x_1, \dots, x_{K}$ with $K=\lceil \ln \rho \rceil$ and $\tau_1, \dots, \tau_{K}$ defined in \Cref{define:theta-k} achieves, with probability at least $1 - 2 \delta$,
\[
F(x_{\kReliable}) - \Fstar \le O\!\left( \sqrt{\ln \frac{\lceil \ln \rho \rceil}{\delta}} \cdot \frac{\Lip \DStar }{\sqrt{n}} +  \ln \frac{\lceil \ln \rho \rceil}{\delta} \cdot \frac{\LipHat \DStar }{n} \right).
\]
\end{corollary}
\begin{proof}
Since the learning rates $R_k = \e^k$ form a geometric grid over $[1, \e^{\lceil \ln \rho \rceil}]$, there exists some $k^* \in \{1, \dots, \lceil \ln \rho \rceil \}$ with $\DStar \le R_{k^*} \le \e \cdot \DStar$. By \Cref{thm:adaptive-SGD}, for this $k^*$, with probability $1 - \delta$,
\[
F(x_{k^*}) - \Fstar\le O\!\left( \frac{\Lip R_{k^*} \sqrt{\ln(2/\delta)}}{\sqrt{n}} \right) = O\!\left( \frac{\Lip \DStar \sqrt{\ln(2/\delta)}}{\sqrt{n}} \right).
\]
Applying \Cref{eq:F-rely-x-k-guarantee} of \Cref{thm:reliable-guarantees} with $k = k^*$ and $\| x_{k^*} \| \le R_{k^*} \le \e \cdot \DStar$, and then union-bounding gives that with probability at least $1 - 2 \delta$
\[
F(x_{\kReliable}) - \Fstar \le F(x_{k^*}) - \Fstar + (1+\gamma)\left( \sqrt{2\tilde{c}} \frac{\Lip \| x_{k^*} \|}{\sqrt{n\edit{-1}}} + \frac{14 \tilde{c}}{3} \cdot \frac{\LipHat \| x_{k^*} \|}{n-1} \right) \le O\!\left( \sqrt{\tilde{c}} \frac{\Lip \DStar}{\sqrt{n}} + \tilde{c} \frac{\LipHat \DStar}{n-1} \right).
\]
Since our number of candidates $K = \lceil \ln \rho \rceil$, we have $\tilde{c} = \ln \frac{4 \lceil \ln \rho \rceil}{\delta}$.
\end{proof}

If the upper bound $\LipHat$ on the Lipschitz constant is within a factor of $\tilde{O}(\sqrt{n})$ of $L$ then the first term dominates and this bound matches the optimal known-parameter sample complexity of $O(\Lip \DStar / \sqrt{n})$ up to $\log\log$ factors.} This sample complexity is slightly better than 
the sample complexity achieved by \citet{carmon2022making} as the $\log\log$ factors are of lower order. 
One disadvantage of this approach is that it requires $O(n \ln \rho )$ gradient evaluations thus making its computational complexity worse than \citet{carmon2022making}.
Another disadvantage of this approach is that it requires knowing a range of values that the distance to optimality will lie in, i.e., $[1, \rho]$. This issue will be resolved by our method presented in \Cref{sec:adaptivity-for-free}.

Finally, \Cref{eq:F-rely-strong-convexity-guarantee} shows that \Cref{alg:reliable-hyperparameter-selection} also automatically adapts to strong convexity, thus preserving the benefits of standard model selection (\Cref{prop:greedy-strong-convexity}).
However, \Cref{prop:greedy-strong-convexity} is slightly stronger because $(i)$ it does not require an estimate of the Lipschitz constant $\LipHat$ and $(ii)$ it applies for all $\delta \in (0,1)$ whereas \Cref{alg:reliable-hyperparameter-selection} applies for a pre-specified $\delta$.

\mysection{Perfect adaptivity to unknown distance to optimality}{Perfect Adaptivity to Unknown Distance to Optimality}\label{sec:adaptivity-for-free}

This section develops a method with perfect adaptivity to unknown distance to optimality, matching optimal known-parameter high-probability bounds. Our results hold for $p$-norms with three values of $p$: 1, 2, and $\infty$. 
While our method requires an estimate of the Lipschitz constant (Assumption~\ref{assume:Lip-estimates}), we show that for a sample of size $n$, an estimate within $\tilde{O}(\sqrt{n})$ of the true Lipschitz constant suffices to match optimal known-parameter high-probability bounds.

\begin{assumption}[Upper bound on Lipschitz constants]
	\label{assume:Lip-estimates}
Suppose that $f$ is differentiable, convex in $x$, and satisfies $\| \grad f(x;S) \|_2 \le \Lip \le \LipHat$ and $\abs{ \grad _j f(x;S)} \le \LipCoord_j \le \hat{\LipCoord}_j$ almost surely for all $j \le d$ and $x \in \X$,  where $\LipHat$ and $\hat{\LipCoord} $ are known.
\end{assumption}

Assumption~\ref{assume:Lip-estimates} treats $f$ as differentiable only for the sake of simplicity: our results extend to nondifferentiable functions via a smoothing argument, e.g., using the Moreau envelope \cite{moreau1965proximity}. %

Our optimal adaptive method consists of two stages. First, using half our sample, we apply regularized empirical risk minimization (ERM) to find a radius $R$ such that, with high probability, the minimum of $F$ in a ball of radius $R$ is close to its global minimum, and also $R=O(\DStar)$. Second, using the other half of our sample, we apply an off-the-shelf optimal known-parameter stochastic optimization method to approximately minimize $F$ in a ball of radius $R$. Since $R$ is given to the method, it need not be adaptive to the distance to optimality. Our approach is summarized in \Cref{alg:PUD}.

\begin{algorithm}
	\begin{algorithmic}[1]
		\INPUT Sample functions 
		$f_1, \dots, f_{2n}$ 
		\INPUT Stochastic optimization algorithm $\alg$ mapping a ball radius and $n$ sample functions to an approximate minimizer of $F$ in the ball
		\vspace{0.09cm}
		\STATE Compute $\lambda[]$ from $\gradDiff \coloneqq \frac{1}{n}\sum_{i=1}^n \coor{\grad f_i(x) - \grad \bar{F}(x)}^2$ (see \Cref{table:combinations-that-satisfy-general-adaptivity-for-free})
		\vspace{0.09cm}
		\STATE Choose $\xReg{\lambda[]} \in \argmin_{x \in \X} \frac{1}{n} \sum_{i=1}^n f_i(x) + \lambda[] \| x \|$ 
		\vspace{0.09cm}
		\STATE Set $\OutputPerfect \gets \alg( 2 \| \xReg{\lambda[]} \|;  f_{n+1}, \dots, f_{2n})$
		\OUTPUT $\OutputPerfect$
	\end{algorithmic}
	\caption{\textsc{OptimalAdaptiveMethod}}\label{alg:PUD}
\end{algorithm}

\myparagraph{Stage 1: Localization via ERM}
The key component of our approach is (non-squared) norm regularization. For a given regularization parameter $\lambda[]$, consider the following 
regularized ERM and regularized population risk minimizers:
\[
\xReg{\lambda[]} \in \argmin_{x \in \X}\left\{ \bar{F}(x) + \lambda[] \| x \| \right\}\text{ and } \xRegStar{\lambda[]} \in \argmin_{x \in \X}  \left\{F(x) + \lambda[] \| x \| \right\}.
\]
To obtain our results, it will suffice to choose $\lambda[]$ such that the following condition holds.
\begin{condition}\label{condition:sufficient-lambda}
Let $f$ be differentiable. Suppose $\delta \in (0,1)$ and $\lambda[] > 0$ are (known) constants such that
$\P\big( \big\| \grad \bar{F}(\xRegStar{\zeta}) - \grad F(\xRegStar{\zeta}) \big\|_{*} > \lambda[] / 2 \big) \le \delta$ for $\zeta \in \{\frac{\lambda[]}{3}, 3\lambda[]\}$.
\end{condition}

\noindent 
The following lemma\ShortPaper{, which we prove in \Cref{app:proof-bound-regularization-direction},} is the key implication of \Cref{condition:sufficient-lambda}.
\begin{lemma}\label{lem:bound-regularization-direction}
	If \Cref{condition:sufficient-lambda} holds then $\P(  \| \xRegStar{3 \lambda[]} \| \le 2 \| \xReg{\lambda[]} \| \le  14 \| \xRegStar{\lambda[]/3} \| ) \ge 1 - 2 \delta$.
\end{lemma}

\begin{proof}
The remainder of this proof will assume that
\begin{subequations}\label{eq:grad-F-minus-F-bar-subeqs}
\begin{flalign}
\big\| \grad \bar{F}(\xRegStar{\lambda[]/3}) - \grad F(\xRegStar{\lambda[]/3}) \big\|_{*} \le \lambda[] / 2 \label{assume:grad-F-bar-is-close-to-F-third-lambda} \\
\big\| \grad \bar{F}(\xRegStar{3 \lambda[]}) - \grad F(\xRegStar{3 \lambda[]}) \big\|_{*} \le \lambda[] / 2 \label{assume:grad-F-bar-is-close-to-F-3-lambda}
\end{flalign}
\end{subequations}
and under this assumption we will show that $\| \xRegStar{3 \lambda[]} \| \le 2 \| \xReg{\lambda[]} \| \le  14 \| \xRegStar{\lambda[]/3} \|$.
Proving this implication suffices to prove the lemma because 
\Cref{condition:sufficient-lambda} and a union bound implies \Cref{eq:grad-F-minus-F-bar-subeqs} holds with probability $1-2\delta$.

We will find the following well-known fact useful \cite[Theorem 1.5.1, Section~1.5]{oden1980theory}: if $H$ is convex and differentiable, $\alpha>0$, and
$z\in\argmin_{x\in\X}\{H(x)+\alpha\|x\|\}$, then for any $y\in\X$,
\begin{flalign}\label{eq:constrained-regularized-optimality}
\grad H(z)\cdot (y-z) \ge \alpha(\|z\| - \|y\|).
\end{flalign}
First, we will show that \Cref{assume:grad-F-bar-is-close-to-F-third-lambda} implies
\begin{flalign}\label{eq:upper-bound-x-hat-lambda-x-star-lambda-third}
\| \xReg{\lambda[]} \| \le  7 \| \xRegStar{\lambda[]/3} \|.
\end{flalign}
Let $q := \grad \bar{F}(\xRegStar{\lambda[]/3}) - \grad F(\xRegStar{\lambda[]/3})$.
By $(i)$ definition of $\xReg{\lambda[]}$, $(ii)$ convexity of $\bar F$, $(iii)$ \Cref{eq:constrained-regularized-optimality} applied with
$H=F$, $z=\xRegStar{\lambda[]/3}$, $y=\xReg{\lambda[]}$, and $\alpha=\lambda[]/3$,
and $(iv)$ $\| q \|_{*} \le \lambda[] / 2$ from \Cref{assume:grad-F-bar-is-close-to-F-third-lambda}, H\"older's inequality, and the triangle inequality, we have
\begin{flalign*}
\bar{F}(\xRegStar{\lambda[]/3}) + \lambda[] \| \xRegStar{\lambda[]/3} \|
&\us{\ge}{i} \bar{F}(\xReg{\lambda[]}) + \lambda[] \| \xReg{\lambda[]} \| \\
&\us{\ge}{ii}  \bar{F}(\xRegStar{\lambda[]/3}) + \grad \bar{F} (\xRegStar{\lambda[]/3}) \cdot ( \xReg{\lambda[]} - \xRegStar{\lambda[]/3}) + \lambda[] \| \xReg{\lambda[]} \| \\
&=  \bar{F}(\xRegStar{\lambda[]/3}) + \grad F (\xRegStar{\lambda[]/3}) \cdot ( \xReg{\lambda[]} - \xRegStar{\lambda[]/3}) + q \cdot ( \xReg{\lambda[]} - \xRegStar{\lambda[]/3})
+ \lambda[] \| \xReg{\lambda[]} \| \\
&\us{\ge}{iii}  \bar{F}(\xRegStar{\lambda[]/3}) + \frac{\lambda[]}{3}\left(\|\xRegStar{\lambda[]/3}\|- \| \xReg{\lambda[]} \|\right) + q \cdot ( \xReg{\lambda[]} - \xRegStar{\lambda[]/3})
+ \lambda[] \| \xReg{\lambda[]} \| \\
&\us{\ge}{iv}  \bar{F}(\xRegStar{\lambda[]/3}) + \frac{\lambda[]}{3} (\|\xRegStar{\lambda[]/3}\|- \| \xReg{\lambda[]} \| ) - 
\frac{\lambda[]}{2} (\| \xReg{\lambda[]} \|+\|\xRegStar{\lambda[]/3}\| )
+ \lambda[] \| \xReg{\lambda[]} \|.
\end{flalign*}
Rearranging gives $\| \xReg{\lambda[]}\| \le 7\|\xRegStar{\lambda[]/3}\|$,
which implies \Cref{eq:upper-bound-x-hat-lambda-x-star-lambda-third}.

The remainder of the proof will show that \Cref{assume:grad-F-bar-is-close-to-F-3-lambda} implies
\begin{flalign}\label{eq:x-lambda-is-not-small}
 \| \xRegStar{3\lambda[]} \| \le 2 \| \xReg{\lambda[]} \|.
\end{flalign}
Let $s := \grad \bar{F}(\xRegStar{3 \lambda[]}) - \grad F(\xRegStar{3 \lambda[]})$.
By $(i)$ convexity of $\bar F$, $(ii)$ \Cref{eq:constrained-regularized-optimality} applied with
$H=F$, $z=\xRegStar{3\lambda[]}$, $y=\xReg{\lambda[]}$, and $\alpha=3\lambda[]$, and $(iii)$ convexity of $\bar F$, we get
\begin{flalign*}
\bar{F}(\xReg{\lambda[]}) +3 \lambda[] \| \xReg{\lambda[]} \|
&\us{\ge}{i} \bar{F}(\xRegStar{3 \lambda[]}) 
+ \grad \bar F(\xRegStar{3 \lambda[]}) \cdot (\xReg{\lambda[]} - \xRegStar{3 \lambda[]})
+3 \lambda[] \| \xReg{\lambda[]} \| \\
&= \bar{F}(\xRegStar{3 \lambda[]}) 
+ \grad F(\xRegStar{3 \lambda[]}) \cdot (\xReg{\lambda[]} - \xRegStar{3 \lambda[]}) + s\cdot (\xReg{\lambda[]} - \xRegStar{3 \lambda[]})
+3 \lambda[] \| \xReg{\lambda[]} \| \\
&\us{\ge}{ii} \bar{F}(\xRegStar{3 \lambda[]})
+3 \lambda[] \| \xRegStar{3 \lambda[]} \|
+ s\cdot (\xReg{\lambda[]} - \xRegStar{3 \lambda[]}) \\
&\us{\ge}{iii} \bar{F}(\xReg{\lambda[]}) + (\grad \bar{F}(\xReg{\lambda[]}) - s) \cdot (\xRegStar{3 \lambda[]} - \xReg{\lambda[]})
+3 \lambda[] \| \xRegStar{3 \lambda[]} \| 
\end{flalign*}
Next, by $(i)$ rearranging the previous display,
$(ii)$ 
\Cref{eq:constrained-regularized-optimality} applied with
$H=\bar{F}$, $z=\xReg{\lambda[]}$, $y=\xRegStar{3\lambda[]}$, and $\alpha=\lambda[]$, and $(iii)$
H\"older's inequality, the triangle inequality, and $\| s \|_{*} \le \lambda[]/2$ by \Cref{assume:grad-F-bar-is-close-to-F-3-lambda}, we get
\begin{flalign*}
\| \xReg{\lambda[]} \|
& \us{\ge}{i}  \| \xRegStar{3 \lambda[]} \|
+  \frac{(\grad \bar{F}(\xReg{\lambda[]}) - s) \cdot (\xRegStar{3 \lambda[]} - \xReg{\lambda[]})}{3 \lambda[]} \\
& \us{\ge}{ii} \| \xRegStar{3 \lambda[]} \|
+ \frac{\| \xReg{\lambda[]} \| - \| \xRegStar{3\lambda[]} \|}{3}
-  \frac{s \cdot (\xRegStar{3 \lambda[]} - \xReg{\lambda[]})}{3 \lambda[]} \\
&\us{\ge}{iii}  \| \xRegStar{3 \lambda[]} \|
+ \frac{\| \xReg{\lambda[]} \| - \| \xRegStar{3\lambda[]} \|}{3} -  \frac{\| \xReg{\lambda[]} \| + \| \xRegStar{3 \lambda[]} \|}{6}.
\end{flalign*}
Rearranging yields $\| \xRegStar{3\lambda[]} \| \le \frac{5}{3}\|\xReg{\lambda[]}\| \le 2\|\xReg{\lambda[]}\|$,
which proves \Cref{eq:x-lambda-is-not-small}. Combining this with
\Cref{eq:upper-bound-x-hat-lambda-x-star-lambda-third} proves the claim.
\end{proof}

Let us see why the bound $ \| \xRegStar{3 \lambda[]} \| \le 2\| \xReg{\lambda[]} \| \le  14 \| \xRegStar{\lambda[]/3}\| $ establishes that $R= 2\| \xReg{\lambda[]}\|$ is a valid output for the first stage of our method. First, since $\| \xRegStar{3 \lambda[]} \| \le R$ we have
\begin{equation}\label{eq:erm-ball-suboptimality}
\min_{x\in\X:\norm{x}\le R} F(x) \le F( \xRegStar{3 \lambda[]} )\le \Fstar +3 \lambda[](\|\xstar\|-\| \xRegStar{3 \lambda[]} \|) \le \Fstar +3 \lambda[] \DStar,
\end{equation}
where the second inequality is due to the definition of $\xRegStar{3 \lambda[]}$ as a minimizer of $x\mapsto F(x) + 3\lambda[] \|x\|$. Second, we have $R \le 14 \| \xRegStar{\lambda[]/3}\|  \le 14\DStar$, implying that $R=O(\DStar)$ as required. 

It remains to find values of $\lambda[]$ such that \Cref{condition:sufficient-lambda} holds and the optimality gap $3 \lambda[] \|\xstar\|$ is sufficiently small. When $\Lip$ is known exactly, this is straightforward: standard concentration inequalities for the sum of bounded vectors \cite[Corollary 10a]{howard2020time} %
imply that \Cref{condition:sufficient-lambda} holds for $\lambda[] = O(L\sqrt{\ln(1/\delta) / n} )$. However, when we only have bounds on the Lipschitz constant as in Assumption~\ref{assume:Lip-estimates}, we must estimate $\lambda[]$ from the empirical gradient variance. The following lemma provides novel empirical vector concentration bounds\footnote{We note in passing that \Cref{lem:vector-concentration-bounds} yields that $\P\big( \big\| \grad \bar{F}(x) - \grad F(x) \big\|_{*} > \lambda[] \big) \le \delta$ for \emph{any} fixed $x \in \X$, but our results only require it for the two values used in \Cref{condition:sufficient-lambda}.} which establish \Cref{condition:sufficient-lambda} when substituting $V_i=\grad f_i(x)$, $\bar{V} =\grad \bar{F}(x)$, $C_j = \hat{\LipCoord}_j$ and $C = \LipHat$. \Cref{table:combinations-that-satisfy-general-adaptivity-for-free} summarizes the  resulting choices of $\lambda[]$.
\begin{lemma}\label{lem:vector-concentration-bounds}
	Let $V_1, \dots, V_n$ be a sequence of i.i.d. random vectors in $\R^d$. Define $\nu := \E[V_i]$, and $\bar{V} := \frac{1}{n} \sum_{i=1}^n V_i$.
	Then:
	\begin{enumerate}
		\item  \label{lem:vector-concentration-bounds:1a} If $\| V_i \|_2 \le \Vub$ almost surely where $\Vub$ is a constant, then for all $\delta \in (0,1)$, \\
		\indent\quad~\quad$\P\Big( \| \bar{V} - \nu \|_2 > \frac{2 \sqrt{\sum_{i=1}^n \| V_i - \bar{V} \|_2^2 \ln \frac{6}{\delta}}}{n} +  \frac{10 \Vub \ln \frac{6}{\delta}}{n-1}  \Big) \le \delta.$
		\item If $\abs{ \coor{ V_i } }\le \Vub_j$ for all $j \in [d]$ almost surely where $\Vub_j$ is a constant, then for all $\delta \in (0,1)$,
		\vspace{-0.2cm}
		\begin{enumerate}
			\item \label{lem:vector-concentration-bounds:2b}
			$\P\Big( \| \bar{V} - \nu \|_{\infty} >  \max_{j \in [d]}  \sqrt{\frac{2  \sum_{i=1}^n \coor{V_i - \bar{V} }^2 \ln \frac{4 d}{\delta} }{n(n-1)}} + \frac{14 \Vub_j \ln \frac{4 d}{\delta}}{3 (n-1)}  \Big) \le \delta$,
			\item \label{lem:vector-concentration-bounds:2d}
			$\P\Big( \| \bar{V} - \nu \|_1 > \sum_{j=1}^d \frac{9}{4} \sqrt{\frac{2 \sum_{i=1}^n \coor{V_i - \bar{V} }^2 \ln \frac{\edit{18}}{\delta}}{n(n-1)}} + \frac{24 \Vub_j \ln \frac{18}{\delta}}{n-1} \Big) \le \delta$.
		\end{enumerate}
	\end{enumerate} 	
\end{lemma}

The proof of \Cref{lem:vector-concentration-bounds} appears in \Cref{app:concentration-inequalities}.
\Cref{lem:vector-concentration-bounds}.\ref{lem:vector-concentration-bounds:1a} applies the ideas of \citet{maurer2009empirical} to replace the population variance in \citet[Corollary 10b]{howard2020time} with the sample variance.
\Cref{lem:vector-concentration-bounds}.\ref{lem:vector-concentration-bounds:2b} is a straightforward application of a union bound to standard concentration inequalities.
\Cref{lem:vector-concentration-bounds}.\ref{lem:vector-concentration-bounds:2d} is a novel concentration bound of independent interest. For example, if 
one coordinate of $V_i$ is a Rademacher random variable and the others are zero almost surely, then \Cref{lem:vector-concentration-bounds}.\ref{lem:vector-concentration-bounds:2d} shows that with constant probability $\| \bar{V} \|_1 \le O(\sqrt{1 / n})$ but \Cref{lem:vector-concentration-bounds}.\ref{lem:vector-concentration-bounds:1a} only shows 
$\| \bar{V} \|_2 \le O(\sqrt{1 / n})$. \citet[Corollary 23]{manole2023martingale} provides a similar result, but it 
includes an undesirable dimension-dependence (through the covering number).
The proof of \Cref{lem:vector-concentration-bounds}.\ref{lem:vector-concentration-bounds:2d} hinges on a new lemma which provides a high-probability bound on the sum of \emph{dependent} random variables. 
In particular, the proof of \Cref{lem:vector-concentration-bounds}.\ref{lem:vector-concentration-bounds:2d} applies this lemma to $X_j = \abs{ \coor{\nu - \bar{V}}}$ and uses standard techniques to bound $X_j$ in terms of the variance of $\coor{V_i}$. 

\Cref{lem:dependency-sum-lemma} is tighter than the typical union bound approach which incurs an unnecessary logarithmic dependence on $d$.  The crux of standard concentration bounds is upper bounding $\E[\e^{t \sum_{j=1}^d X_j} ] $ using that $\E[\e^{t \sum_{j=1}^d X_j} ] = \Pi_{j=1}^d \E[\e^{t X_j}]$ where this equality uses independence of $X_1, \dots, X_d$. However, \Cref{lem:dependency-sum-lemma} cannot use this equality because independence is not assumed. The key insight for \Cref{lem:dependency-sum-lemma} is  that due to Jensen's inequality: $\E[\e^{t \sum_{j=1}^d X_j} ] \le  \sum_{j=1}^d w_j \E[ \e^{t X_j / w_j} ]$ where $w$ is a carefully chosen vector from the unit simplex. Moreover, since $\E[\e^{t X_j / w_j}]$ only involves a single random variable, it is straightforward to bound.

\begin{lemma}[Dependent-sum lemma]\label{lem:dependency-sum-lemma}
Let $X_1, \dots, X_d$ be (possibly dependent) random variables, let $a_j, b_j$ be nonnegative constants for all $j \in [d]$\edit{, and let $c \ge 1$}.
Suppose that for all $j \in [d]$ and $\delta \in (0, 1)$, $\P \left( X_j \ge a_j \sqrt{\ln(\edit{c}/\delta)} + b_j \ln (\edit{c} / \delta) \right) \leq \delta$.
Then, for all  $\delta \in (0,1)$, the probability that $\sum_{j=1}^{d} X_j \ge \frac{9}{4} \sum_{j=1}^{d} \left( a_j \sqrt{\ln(6\edit{c} /\delta)} + \edit{2}b_j \ln (6\edit{c} / \delta) \right)$ is at most $\delta$.
\end{lemma}

\begin{proof}
	Let $t$ and $w_j$ be nonnegative constants with $\sum_{j=1}^d w_j = 1$ (their exact values will be selected later). 
	For any $t > 0$, Markov's inequality implies that
	\[
	\P \left( \sum_{j=1}^{d} X_j \geq \gamma \right) = 	\P \left( \exp(t \sum_{j=1}^{d} X_j ) \geq \exp(t \gamma) \right)  \leq \exp(-t \gamma) \E \left[ \exp \left( t \sum_{j=1}^{d} X_j \right) \right].
	\]	
	Since $\sum_{j=1}^d w_j = 1$, by convexity of the exponential function, Jensen's inequality implies that
	\[
	\exp\left(t \sum_{j=1}^{d} X_j\right) = \exp \left(t \sum_{j=1}^{d} w_j \frac{X_j}{w_j} \right) \leq \sum_{j=1}^d w_j \exp \left(t \frac{X_j}{w_j}\right).
	\]
	Taking expectations, using the fact that $w_j$ are constant yields
	\begin{flalign*} 
		\E \left[ \exp(t \sum_{j=1}^{d} X_j ) \right] \leq \sum_{j=1}^{d} w_j \E \left[ \exp(\frac{t X_j}{w_j}) \right].
	\end{flalign*}
	Thus,
	\begin{flalign}\label{eq:prob-bound-in-terms-of-expected-Z_j}
		\P \left( \sum_{j=1}^{d} X_j \geq \gamma \right) \leq \exp(-t \gamma) \sum_{j=1}^{d} w_j \E \left[ \exp(\frac{t X_j}{w_j}) \right].
	\end{flalign} 
	It remains to bound $\E[\exp(t X_j / w_j)]$.
	\ShortPaper{Recall by \Cref{lem:prob_inequality_equivalent_statements}}\LongPaper{By \Cref{lem:prob_inequality_equivalent_statements} (\Cref{app:well-known-concentration-inequalities})} we have
	\[
	\P \left( X_j \ge a_j \sqrt{\ln(\edit{c}/\delta)} + b_j \ln (\edit{c} / \delta) \right) \leq \delta \edit{\implies} \P \left( X_j \geq x \right) \le \edit{c}\exp(- \frac{x^2}{a_j^2 + \edit{2}b_j x}).
	\]
	Let
	$Z_j = \exp(t X_j / w_j)$ then
	\[
	\P(Z_j \ge z_j) = \P(w_j \ln(Z_j)/t \ge w_j \ln(z_j) / t) = \P(X_j \ge w_j \ln(z_j) / t).
	\]
	It follows that
	\[
	\resizebox{\linewidth}{!}{$\displaystyle
	\E[Z_j] = \int_{0}^\infty \P(Z_j \ge z_j) ~ dz_j 
	=  \int_{0}^\infty \P(X_j \ge w_j \ln(z_j) / t) ~ dz_j  
	\le \edit{1 + c} \int_{1}^\infty \exp \left(- \frac{ \ln(z_j)^2 \frac{w_j^2}{t^2} }{a_j^2 + \frac{\edit{2} b_j w_j}{t} \ln(z_j)} \right) \, dz_j .
	$}	
	\]
	Substituting $z_j = \exp(t x_j / w_j)$ gives
	\[
	\E[Z_j] \le \edit{1 + c} \int_{0}^\infty h(x_j) ~ d x_j = \edit{1 + c } \left( \int_{0}^{\hat{x}_j} h(x_j) ~ dx_j + \int_{\hat{x}_j}^{\infty} h(x_j) ~ dx_j \right) 
	\]
	where $h(x_j) := \frac{t}{w_j} \exp \left(\frac{t  x_j}{w_j}- \frac{ x_j^2}{a_j^2 + \edit{2}b_j x_j} \right)$ and $\hat{x}_j = \frac{3 a_j^2}{\edit{2}b_j}$.
	We will now bound each of these terms.
	First,
	\begin{flalign*}
		\int_{0}^{\hat{x}_j} h(x_j) ~ dx_j &\le \int_{0}^{\hat{x}_j}  \frac{t}{w_j} \exp \left( \frac{t x_j}{w_j} -  \frac{1}{4} \left(\frac{x_j}{a_j} \right)^2  \right) dx_j \le 2 \sqrt{\pi} \cdot \frac{t a_j}{w_j}  \exp\left( \frac{t^2 a_j^2}{w_j^2} \right) .
	\end{flalign*}
	Second, if 
	\begin{flalign}\label{eq:if-small-b-j-relative-to-w_j-divided-by-t}
		\frac{t}{w_j} \cdot \frac{\Gamma}{\tau} \leq \frac{1}{\edit{2}b_j} \text{ and } \tau \in (0, 3 \Gamma/4)
	\end{flalign}
	with $\tau$ and $\Gamma$ to be chosen later to satisfy these requirements, then
	by $(i)$ the fact that 
	$x_j / (a_j^2 + \edit{2}b_j x_j) = x_j / (\edit{2} \hat{x}_j b_j / 3 + \edit{2} b_j x_j) \ge \frac{1}{\edit{2} b_j (1/3 + 1)} = \frac{3}{\edit{8} b_j}$, $(ii)$ \Cref{eq:if-small-b-j-relative-to-w_j-divided-by-t}, and $(iii)$ using $\tau \in (0, 3 \Gamma / 4)$   we get 
	\begin{flalign*}
		\int_{\hat{x}_j}^{\infty} h(x_j) ~ dx_j 
		&\us{\le}{i} \int_{\hat{x}_j}^{\infty} \frac{t}{w_j} \exp \left(x_j \left( \frac{t}{w_j} - \frac{3}{\edit{8} b_j} \right) \right) dx_j \us{\le}{ii} \int_{\hat{x}_j}^{\infty}  \frac{t}{w_j} \exp \left(\left( \frac{4 \tau - 3 \Gamma}{4 \tau}\right)\frac{t x_j}{w_j}  \right) dx_j \\
		&\le \int_{0}^{\infty}  \frac{t}{w_j} \exp \left(\left( \frac{4 \tau - 3 \Gamma}{4 \tau}\right)\frac{t x_j}{w_j}  \right) dx_j \us{=}{iii} \frac{4\tau}{3 \Gamma-4\tau}.
	\end{flalign*}
	Combining and substituting into \Cref{eq:prob-bound-in-terms-of-expected-Z_j} gives
	\[
	\resizebox{\linewidth}{!}{$\displaystyle
		\P \left( \sum_{j=1}^{d} X_j \geq \gamma \right) \le \exp(-\gamma t) \sum_{j=1}^d w_j \E[Z_j] \left( \edit{1 +} \frac{4\tau c}{3\Gamma-4\tau} \right) \exp(-\gamma t) + 2 \edit{c} \sqrt{\pi} \sum_{j=1}^d t a_j  \exp\left( \frac{t^2 a_j^2}{w_j^2} - \gamma t \right).
	$}
	\]
	Setting
	$\gamma = \Gamma \sum_{j=1}^d a_j \sqrt{\ln(6\edit{c}/\delta) } + \edit{2}b_j \ln(6\edit{c}/\delta)$ with $\Gamma=\frac{9}{4}$, $t = \tau \frac{\ln(6\edit{c}/\delta)}{\gamma}$ with $\tau = \edit{1.44}$ and $w_j = \Gamma \frac{ a_j \sqrt{\ln(6\edit{c}/\delta) } + \edit{2}b_j \ln(6\edit{c}/\delta)}{\gamma}$ ensures we satisfy \eqref{eq:if-small-b-j-relative-to-w_j-divided-by-t} and $\sum_{j=1}^{d} w_j = 1$. Substituting these values into our bound on $\P \left( \sum_{j=1}^{d} X_j \geq \gamma \right)$ gives
	\begin{flalign*}
		\P \left( \sum_{j=1}^{d} X_j \geq \gamma \right) 
		&\le \resizebox{0.82\linewidth}{!}{$\displaystyle
        \left( \edit{1 +}\frac{4\tau c}{3\Gamma-4\tau} \right) \exp(-\tau \ln(6\edit{c}/\delta)) + 2 \edit{c} \sqrt{\pi} \frac{\tau}{\Gamma} \sqrt{\ln(6\edit{c}/\delta)} \exp\left( \left( \frac{\tau^2}{\Gamma^2} - \tau \right) \ln(6\edit{c}/\delta) \right)
    $} \\
		&= \left( \edit{1 +} \frac{4\tau c}{3\Gamma-4\tau} \right) \left( \frac{\delta}{6\edit{c}} \right)^{\tau} + 2 \edit{c} \sqrt{\pi} \frac{\tau}{\Gamma} \sqrt{\ln(6\edit{c}/\delta)} \left( \frac{\delta}{6\edit{c}} \right)^{\tau - \frac{\tau^2}{\Gamma^2}} \\
		&\le \edit{ (1 + 5.82 c) \left( \frac{\delta}{6c} \right)^{1.44} + 2.27 c \sqrt{\ln(6c/\delta)} \left( \frac{\delta}{6c} \right)^{1.0304} }\\
		&\le \delta
	\end{flalign*}
	\edit{where the last inequality follows because the preceding right-hand side divided by $\delta$, call it $g_c(\delta)$, is at most $1$ on $(0,1)$ for every $c \ge 1$. Term-by-term comparison gives $g_c(\delta) \le g_1(\delta/c)$ as the first-term ratio simplifies to $(1+5.82c)/(6.82c) \le 1$ and the second-term ratio is exactly $1$. Since $\delta/c \in (0,1)$, it suffices that $g_1 \le 1$ on $(0,1)$, which holds because $g_1$ has a unique interior local maximum ${\approx}\,0.93$ (at $\delta \approx 5.4 \times 10^{-7}$) and $g_1(1) \approx 0.996 < 1$.}
\end{proof}

While the values of $\lambda[]$ calculated in \Cref{table:combinations-that-satisfy-general-adaptivity-for-free} facilitate our sample complexity results, they involve a supremum over $\X$ which might be impractical to calculate. An alternative, more practical approach uses our \RMS{} method to grid search over $\lambda[]$. This approach still requires no knowledge of the distance to optimality, but increases the sample complexity by a $\log\log$ factor in our uncertainty in the Lipschitz constant; see \Cref{app:dealing-with-unknown-lambda} for details.

\myparagraph{Stage 2: Constrained stochastic optimization} 
Having reduced the problem to minimization of $F(x)$ in a ball of known radius $R$, our next and final step is to find an approximate minimizer in that ball. It is natural to hope that the ERM solution $ \hat{x}_{\lambda[]}$, which is in the ball by definition, is a suitable approximate minimizer. Unfortunately, even in the Euclidean case with known Lipschitz constant, the best available high probability guarantees for ERM have an additional $\log n$ factor~\cite{bousquet2020sharper}. Therefore, we consider a generic stochastic optimization algorithm $\alg$ that takes in a ball radius $R>0$ and sample functions\footnote{Since in \Cref{alg:PUD} the radius $R$ depends on the first $n$ sample functions, we apply $\alg$ to $n$ additional, fresh samples.} $f_{n+1}, \ldots, f_{2n}$, and outputs a point $\alg(R;f_{n+1}, \ldots, f_{2n})\in \edit{\{ x \in \X : \| x \| \le R \}}$. Assumption~\ref{assume:std-algorithm} parameterizes the approximation guarantee we require of $\alg$.

\begin{assumption}\label{assume:std-algorithm}
\edit{The algorithm $\alg$ satisfies Assumption~\ref{assume:std-algorithm} with parameter $\lambdaAlg > 0$ if, for a given $\delta > 0$ and all $R > 0$, $\alg(R;f_{n+1},\dots,f_{2n})  \in \{ x \in \X : \| x \| \le R \}$ almost surely and}
\[
\P\Big( F(\alg(R;f_{n+1},\dots,f_{2n})) - \min_{x \in \X : \| x \| \le R} F(x) \le \lambdaAlg R  \Big) \ge 1 - \delta.
\]
\end{assumption}
With Assumption~\ref{assume:std-algorithm}, we immediately obtain a generic error bound for \Cref{alg:PUD}.

\begin{lemma}\label{thm:general-adaptivity-for-free}
Let $\OutputPerfect$ be the output of \Cref{alg:PUD}. If $\lambda[]$ satisfies \Cref{condition:sufficient-lambda} and \edit{algorithm $\alg$ satisfies Assumption~\ref{assume:std-algorithm} with parameter $\lambdaAlg$} then, with probability at least $1 - 3 \delta$, $F(\OutputPerfect) - \Fstar \le  (\edit{14}  \lambdaAlg + 3 \lambda[]) \DStar$ and $\| \OutputPerfect \| \le 14 \DStar$.
\end{lemma}

\begin{proof}
	Let $R=2 \| \xReg{ \lambda[]}\|$. 
	By \Cref{lem:bound-regularization-direction}, this lemma's premise, and a union bound, it follows that, with probability at least $1-3\delta$, both  $  \| \xRegStar{3 \lambda[]} \| \le R \le  14 \| \xRegStar{\lambda[]/3} \|$ and $F(\alg(R\edit{; f_{n+1}, \ldots, f_{2n}})) \le \lambdaAlg R + \min_{x \in \X : \| x \| \le R} F(x)$. 
	Upper bounding the latter inequality using \eqref{eq:erm-ball-suboptimality} and $R \le 14 \| \xRegStar{\lambda[]/3} \| \le 14 \DStar$  yields our \edit{desired suboptimality guarantee. The bound $\| \OutputPerfect \| \le 14 \DStar$ follows from $\| \OutputPerfect \| \le R$.}
\end{proof}

It remains to instantiate $\alg$ and $\lambdaAlg$ using specific algorithms; we match each $p$-norm we consider with a different stochastic optimization method, each attaining the minimax optimal high probability bound within a known ball of the corresponding norm, without requiring any knowledge of the Lipschitz constants. For the Euclidean ($p=2$) case, we use  \textsc{AdaSGD}. For $p=1$ we use entropic mirror descent (i.e., mirror descent with KL divergence) \cite{beck2003mirror,nemirovski1983problem} with adaptive step sizes~\cite{orabona2021modern}, which we denote \textsc{AdaEMD}. For $p=\infty$, we use \textsc{AdaGrad} \cite{duchi2011adaptive,mcmahan2010adaptive}. \Cref{table:combinations-that-satisfy-general-adaptivity-for-free} gives the value of $\lambdaAlg$ corresponding to each $p$, and we provide full details in \Cref{app:coro:applications-of-main-theorem}, leading to the following result. \edit{We focus on $p \in \{1, 2, \infty\}$ because these are the most widely studied geometries in stochastic convex optimization \cite{beck2003mirror,nemirovski1983problem,duchi2011adaptive}, each admitting a minimax-optimal parameter-free algorithm.}

\vspace{1em}

\begin{table}[h]
	\centering
	\renewcommand{\arraystretch}{1.5}
	\setlength{\tabcolsep}{3pt}
	\begin{tabular}{ c  c c  c }
		\toprule
		$p$-norm & $\lambda[p]$ & \textbf{$\alg_p$} & $\lambdaAlg[p]$ \\
		\midrule
		$2$ & 
		$\frac{4 \sqrt{\ln \frac{6}{\delta}}}{\sqrt{n}} \sup_{x \in \X} \sqrt{ \sum_{j=1}^d \gradDiff} + \frac{20 \LipHat \ln \frac{6}{\delta}}{n-1}$ & $\textsc{AdaSGD}$ & $O\Big( \frac{\Lip \sqrt{\ln \frac{1}{\delta}}}{\sqrt{n}} \Big)$
		\\
		$1$ & 
		$ \frac{\edit{2} \sqrt{2 \ln \frac{4 d}{\delta} }}{\sqrt{n-1}} \sup_{j \in [d], x \in \X} \sqrt{ \gradDiff }+\frac{28 \| \hat{\LipCoord} \|_{\infty} \ln \frac{4 d}{\delta}}{3 (n-1)}$ 
		& $\textsc{AdaEMD}$ & $O\Big( \frac{\| \LipCoord \|_{\infty} \sqrt{\ln \frac{d}{\delta}}}{\sqrt{n}} \Big)$ 
		\\
		${\infty}$ &
		$\frac{9 \sqrt{2 \ln \frac{18}{\delta}}}{2 \sqrt{n-1}} \sup_{x \in \X} \sum_{j=1}^d \sqrt{\gradDiff}+\frac{\edit{48} \| \hat{\LipCoord} \|_1 \ln \frac{18}{\delta}}{n-1} $ & $\textsc{AdaGrad}$  & 
		$O\Big( \frac{\| \LipCoord \|_{1}  \sqrt{\ln \frac{1}{\delta}}}{\sqrt{n}} \Big)$
		\\
		\bottomrule
	\end{tabular}
	\caption{Settings used in \Cref{thm:instantiates-lambda-phi-table}, where rows correspond to different norms (given in the first column) and $\hat{\LipCoord}$ is defined in Assumption~\ref{assume:Lip-estimates}. In the remaining columns, $\lambda[p]$ is the regularization parameter satisfying Condition~\ref{condition:sufficient-lambda} (for $\gradDiff \coloneqq \frac{1}{n}\sum_{i=1}^n \coor{\grad f_i(x) - \grad \bar{F}(x)}^2$), algorithm \textbf{$\alg$} is an optimal stochastic optimization method in a given norm ball, and $\lambdaAlg[p]$ is the error coefficient for which Assumption~\ref{assume:std-algorithm} holds.}\label{table:combinations-that-satisfy-general-adaptivity-for-free}
\end{table}

\begin{theorem}\label{thm:instantiates-lambda-phi-table}
Suppose that $\delta \in (0,1)$ and Assumption~\ref{assume:Lip-estimates} holds. Then, for $p\in\{1,2,\infty\}$, the output $\OutputPerfect$ of \Cref{alg:PUD} using $\lambda[p], \alg_p$  and $\lambdaAlg[p]$ given in \Cref{table:combinations-that-satisfy-general-adaptivity-for-free} satisfies, with probability $\ge 1-3 \delta$,
\[
F(\OutputPerfect) - \Fstar \le  (\edit{14}  \lambdaAlg[p] + 3 \lambda[p]) \DStar \text{~~and~~}  \| \OutputPerfect \| \le 14 \DStar.
\]
\end{theorem}

\mysubsection{Simultaneously adapting to multiple problem structures}{Simultaneously Adapting to Multiple Problem Structures}\label{sec:multiple-problem-structures}

Finally, we combine our \RMS{} method (\Cref{alg:reliable-hyperparameter-selection}) with \Cref{thm:instantiates-lambda-phi-table} to simultaneously adapt to multiple problem structures. \edit{In
particular, we run \Cref{alg:PUD} with  $p \in \{2, 1, \infty\}$ from \Cref{table:combinations-that-satisfy-general-adaptivity-for-free} and then use \RMS{} to choose the best output.}
The proof and algorithm appear in \Cref{app:example:combine-methods}.

\begin{theorem}\label{example:combine-methods}
Suppose Assumption~\ref{assume:Lip-estimates} holds. Additionally, assume that $\delta\in (0, 1/2]$, $\LipHat \le \Lip  \sqrt{n / \ln(1/\delta)}$ and $\hat{\LipCoord}_j \le \LipCoord_j \sqrt{n / \ln(d/\delta)}$ for all $j \le d$. Then, there exists a parameter-free algorithm that samples $\edit{3}n$ functions and returns $z$ such that, with probability at least $1-\delta$,
\[
F(z) - \Fstar \le \min_{\xstar \in \XStar} O\Bigg( \frac{ \Lip \| \xstar  \|_2  \sqrt{\ln \frac{1}{\delta}}}{\sqrt{n}} \bigwedge \frac{\| \LipCoord \|_{\infty} \| \xstar \|_{1} \sqrt{\ln \frac{d}{\delta}}}{\sqrt{n}} \bigwedge \frac{\| \LipCoord \|_{1} \| \xstar \|_{\infty} \sqrt{\ln \frac{1}{\delta}}}{\sqrt{n}}  \Bigg).
\]
\end{theorem}

\Cref{example:combine-methods} shows that we can---without prior knowledge of $\DStar$---match, up to constant factors, the sample complexity lower bounds for stochastic convex optimization\footnote{The lower bounds for $p\in\{2,\infty\}$ are given by a one-dimensional construction that covers both cases \citep[Proposition 1b]{carmon2024price}. 
 The lower bound for the $p=1$ norm is implied by \citet[Proposition 1b]{carmon2024price} for $d < 1/\delta$ and by \citet[Theorem 3]{levy2019necessary} and Markov's inequality for $d \ge 1/\delta$.}
 with a known distance to optimality across three standard geometries. 
A limitation of \Cref{example:combine-methods} is that it requires an upper bound on the Lipschitz constants that is tight up to a factor of $\tilde{O}(\sqrt{n})$. This is unavoidable due to a lower bound on the sample complexity of parameter-free stochastic optimization \cite[Theorem 3]{carmon2024price}.

\section{Experiments with \RMS{}}\label{sec:experiments}

\newcommand{\MaxDiff}[1]{M(#1)}

While the primary contribution of this paper is theoretical, we provide two experiments to demonstrate that our \RMS{} method can be useful in practice.
This section provides an overview of them; complete details are provided in \Cref{sec:further-experiment-details}.
For both experiments we run our \RMS{} method (\Cref{alg:reliable-hyperparameter-selection}) with $\gamma=3$ and $\tau_k = \tfrac{1}{2} ( \sqrt{{\SampleVar{k}} / n} + \MaxDiff{x_k} / n )$
where $\MaxDiff{x}$ satisfies $\abs{f(x;S) - f(x_0;S)} \le \MaxDiff{x}$, $\forall S \in \SS$, $x \in \X$. 
This choice of $\tau_k$ is inspired by $\tau_k = \sqrt{2 \tilde{c} \SampleVar{k} / n} +  14 \tilde{c} \MaxDiff{x_k} / (3 n)$ where $\tilde{c} := \ln \frac{4 K}{\delta}$
which is a more conservative choice satisfying\footnote{By a union bound and \Cref{thm:empirical-bennett} in \Cref{app:well-known-concentration-inequalities} with $Z_i = f_i(x_k)-f_i(x_0)$, $a=-\MaxDiff{x_k}$, and $b=\MaxDiff{x_k}$.} \Cref{assume:confidence-intervals}. The function $\MaxDiff{x}$ is tailored to each of our experiments.

\myparagraph{Few-shot learning with CIFAR-10 and CLIP}  We fine-tune the last layer of a zero-shot CLIP model \cite{radford2021learning} with cross-entropy loss on CIFAR-10 \cite{krizhevsky2009learning} using \textsc{AdaGrad} with batch size of $\min\{ 0.5 \times \textsc{train set size}, 40 \}$. We search over hyperparameters in a grid over the number of 
epochs (10, 20, 30, or 40), 
the learning rate ($\eta \in \{ 4^{-10}, 4^{-9}, \dots, 4^{5} \}$),
and two calibration parameters ($\omega_1 \in \{ 0.33, 0.66, 1 \}$ and $\omega_2 \in \{ 0.5, 1.0, 2.0 \}$) which allow us to consider the combination of weights $\omega_2 ( (1- \omega_1) x + \omega_1 x_0)$  where $x$ is one of the trained model weights. This gives a total of $1 + 4 \times 16 \times 3 \times 3= 577$ hyperparameter combinations including the zero-shot model. %
We use the same number of shots for the training and validation sets.
We apply the \RMS{} method to minimize the cross-entropy loss on the validation set with $\MaxDiff{x} := 2 \| x - x_0 \|_{2,\infty}$ where $x_k$ is the weight matrix corresponding to each hyperparameter combination and $x_0$ is the initial zero-shot model; \Cref{sec:justify-norm} justifies this choice of $\MaxDiff{x}$.
The left plot in \Cref{fig:combined-experiment-plots} shows that, when the validation set is small (with 32 shots or fewer), then standard model selection is worse than the zero-shot model, while our method successfully improves upon it;
when the validation set is large (64 shots or more), both methods exhibit almost identical performance.
 
 \begin{figure}[t]
 	\includegraphics[width=0.495\textwidth]{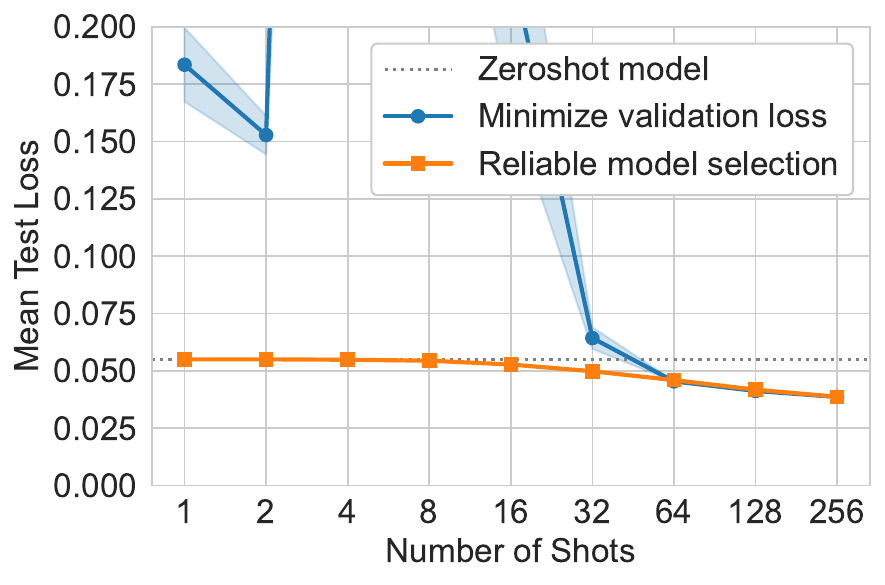}
	\includegraphics[width=0.495\textwidth]{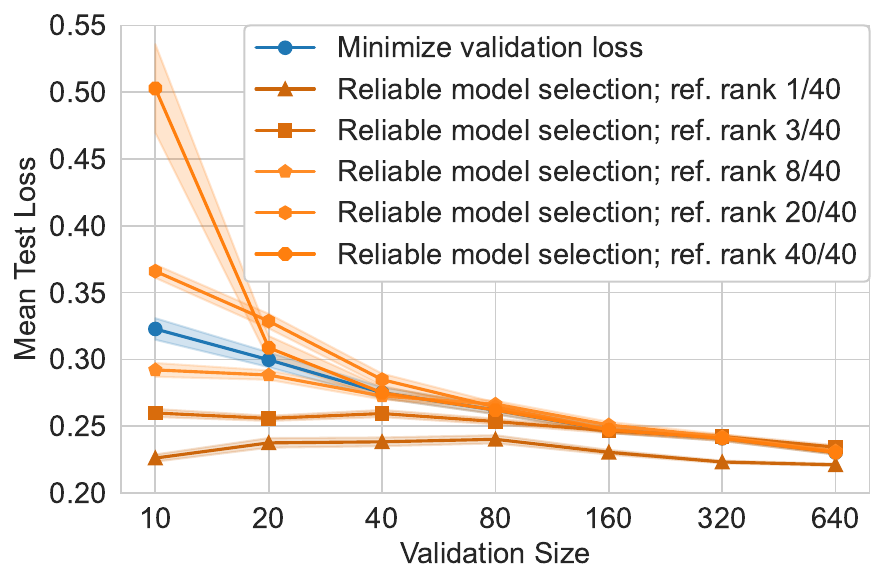}
	\vspace{-0.2cm}
 	\caption{\textbf{Left:} CLIP model fine-tuning of the last layer
 		for ViT-L/14 starting from the zero-shot weights
 		on CIFAR-10. 
		\textbf{Right:} Prompt engineering a large language model, \texttt{gemini-1.5-flash-002}, on the task of counting the number of shapes in images.
        Both experiments are based on $200$ runs, each of which resamples the training and validation sets. Shaded regions represent one standard error.
 		}\label{fig:combined-experiment-plots}
 \end{figure}

\myparagraph{Prompt engineering Gemini} We ask \texttt{gemini-1.5-flash-002} \cite{team2024gemini}, using 40 different prompts, to count the number of shapes in an image. The task is to select the prompt that yields the lowest absolute error between the LLM-reported and actual number of shapes. 
We randomly generated 5,000 images containing between 1 and 12 shapes while varying other characteristics such as color, size, and background. 
We project all predictions to the set $[0,12]$ and thus set $\MaxDiff{x} := 12$.
As the reference model, we used the best, third-best, 20th percentile, median, and worst prompts as evaluated on all 5,000 images. 
The right plot in \Cref{fig:combined-experiment-plots} shows that for a good choice of reference model, \RMS{} outperforms standard model selection, especially with smaller validation sizes;
the reverse holds for a poor reference model. 
When the validation set is sufficiently large, both methods perform similarly, regardless of the choice of reference model.

\myparagraph{Limitations}
Our method depends on both knowledge of the problem structure to compute $M(x)$ and a high-quality reference model,
which is often not possible.
Also, our approach is likely only useful for small validation sets because overfitting is rare for large validation sets \cite{roelofs2019meta,recht2019imagenet,yadav2019cold,miller2020effect}.

\newcommand{\acknowledgements}[0]{We thank Aaditya Ramdas, Cristobal Guzman, and John Duchi for helpful input. We also thank 
Itai Kreisler and Akif Khan for their helpful feedback on the paper.

This work was supported by the NSF-BSF program, under NSF grant \#2239527
and BSF grant \#2022663.
OH acknowledges support from
AFOSR grant \#FA9550-23-1-0242.
YC acknowledges support from the Israeli Science
Foundation (ISF) grant no. 2486/21, the Alon Fellowship, and the Adelis Foundation.
HB and AK were supported by the 
Allias/Holzman Undergraduate Research Award from 
the Department of Industrial Engineering at the University of Pittsburgh.
HB was also supported by an REU supplement to NSF grant \#2239527.
This research was supported in part by the University of Pittsburgh Center for
Research Computing and Data, RRID:SCR\_022735, through the resources provided. Specifically, this work used the H2P cluster, which is supported by NSF award number OAC-2117681.
}

\arxiv{
\section{Acknowledgements}	
\acknowledgements
}

\jmlr{
\acks{\acknowledgements}
}

\arxiv{\bibliographystyle{plainnat}}

\arxiv{\bibliography{bib.bib}}

\appendix

\mysection{Well-known results}{Well-Known Results}\label{app:well-known-concentration-inequalities}

For \Cref{thm:hoeffdings-inequality}, \Cref{thm:bennett-inequality}, and \Cref{thm:empirical-bennett}, let $Z, Z_1, \dots, Z_n$ be i.i.d. random variables with values in $[a, b]$ where $a$ and $b$ are constants. Let $\bar{Z}_n = \frac{1}{n} \sum_{i=1}^{n} Z_i$. Assume $\delta \in (0, 1)$.

\begin{theorem}[Hoeffding's inequality \cite{Hoeffding1963}]\label{thm:hoeffdings-inequality}
	With probability at least $1 - \delta$,
	\[
		\abs{\E Z - \bar{Z}_n} \le (b-a) \sqrt{\frac{\ln 2/\delta}{2 n}}.
	\]
\end{theorem}

\begin{theorem}[Bennett's inequality \cite{bennett1962probability}]\label{thm:bennett-inequality}
	Let $\sigma^2$ be the variance of $Z$. With probability at least $1 - \delta$,
	\[
		\abs{\E Z - \bar{Z}_n} \le \sigma \sqrt{\frac{2\ln 2/\delta}{n}} + \frac{(b - a) \ln 2/\delta}{3 n}.
	\]
\end{theorem}

\begin{theorem}[Theorem 4 of \citet{maurer2009empirical}]\label{thm:empirical-bennett}
	With probability at least $1 - \delta$,
	\[
		\E Z - \bar{Z}_n \le \sqrt{\frac{2 s_n^2 \ln 2/\delta }{n}} + \frac{7 (b - a)\ln 2/\delta}{3 (n-1)}
	\]
	 	with the empirical variance $s_n^2 = \frac{1}{n-1} \sum_{i=1}^n (Z_i - \bar{Z}_n)^2$.
\end{theorem}

Note that by a union bound, \Cref{thm:empirical-bennett} immediately implies that with probability at least $1-\delta$
\[
	\abs{\E Z - \bar{Z}_n} \le \sqrt{\frac{2 s_n^2 \ln 4/\delta }{n}} + \frac{7 (b - a)\ln 4/\delta}{3 (n-1)}.
\]

\edit{
\begin{lemma}[Tail Inversion] \label{lem:prob_inequality_equivalent_statements}
    Let $X$ be a random variable. Let $a, b > 0$ and $c \ge 1$. For all $x \ge 0$ and $\delta \in (0, 1)$, 
    \[ \P \left( X \ge a \sqrt{\ln(c/\delta)} + b \ln(c/\delta) \right) \le \delta \iff \P(X \ge x) \le c \exp\left(-\frac{2x^2}{a^2 + 2bx + a\sqrt{a^2 + 4bx}}\right). \]
    Consequently, the following simplified bounds hold
    \begin{itemize}
        \item[(i)] If $\P(X \ge x) \le c \exp\left(-\frac{x^2}{a^2 + bx}\right)$, then $\P \left( X \ge a \sqrt{\ln(c/\delta)} + b \ln(c/\delta) \right) \le \delta$.
        \item[(ii)] If $\P \left( X \ge a \sqrt{\ln(c/\delta)} + b \ln(c/\delta) \right) \le \delta$, then $\P(X \ge x) \le c \exp\left(-\frac{x^2}{a^2 + 2bx}\right)$.
    \end{itemize}
\end{lemma}

\begin{proof}
    For $x > 0$, solving the quadratic $x = a\sqrt{\ln(c/\delta)} + b\ln(c/\delta)$ for $\sqrt{\ln(c/\delta)}$ yields
    $$\sqrt{\ln(c/\delta)} = \frac{-a + \sqrt{a^2 + 4bx}}{2b} = \frac{2x}{a + \sqrt{a^2 + 4bx}}$$
    Squaring both sides gives
    $$\ln(c/\delta) = \frac{2x^2}{a^2 + 2bx + a\sqrt{a^2 + 4bx}}$$
    Rearranging this directly yields $\delta = c \exp\left(-\frac{2x^2}{a^2 + 2bx + a\sqrt{a^2 + 4bx}}\right)$, which establishes the equivalence.

    To prove the simplified bounds, we bound the denominator $a^2 + 2bx + a\sqrt{a^2 + 4bx}$ inside the exponential. 
    For (i), observing that $a \le \sqrt{a^2 + 4bx}$ lower-bounds the denominator by $2(a^2 + bx)$ yields the result. 
    For (ii), using the inequality $\sqrt{1+z} \le 1 + z/2$ for $z = 4bx/a^2$, we have $a\sqrt{a^2+4bx} \le a^2 + 2bx$. This upper-bounds the denominator by $2(a^2 + 2bx)$, completing the proof.
\end{proof}
}

\begin{lemma}
If $F : \X \rightarrow \R$ is $\mu$-strongly convex, then
\begin{flalign}\label{eq:useful-strong-convexity-implication}
\frac{\mu}{2} \| x - \xstar \|^2 \le F(x) - F(\xstar) \quad \forall x \in \X
\end{flalign}
where $\xstar$ is the unique minimizer of $F$.
\end{lemma}
\begin{proof}
	If $F : \X \rightarrow \R$ is $\mu$-strongly convex, then
	\[
	\frac{\mu}{2} \| u - v \|^2 + \grad F(v) \cdot (u - v) \le F(u) - F(v) \quad \forall u, v \in \X.
	\]
	Using $v = \xstar$ and $u = x$ where $\xstar$ is the unique minimizer of $F$ gives the required bound.
\end{proof}

\section{\edit{Proof of \Cref{prop:greedy-model-selection-performs-poorly}}}\label{app:standard-model-selection-slow}

Let $S_1, S_2,\dots, S_{2n}$ be independent Bernoulli random variables with success probability
of $q = \frac{1}{2} - \frac{1}{16 \sqrt{n}}$ with
\[
\GreedyPoorSetup.
\]
We will also let $\eta_{\max} = \max_{k \in [K]} \eta_k$.
\begin{lemma}\label{lem:anti-concentration-bound}
\begin{subequations}
If $n \ge 3000$ then 
\begin{flalign}\label{lem:validation-is-not-representitive} 
\P\left( \sum_{i=1}^n S_i  >\frac{n}{2} \right) \ge 0.44
\end{flalign}
and
\begin{flalign}\label{eq:concentrate-binomial-distribution}
\P\left(  \sum_{i=1}^n S_i  > \frac{n}{2} + \frac{4}{7} \sqrt{2 n}  \right)  \ge 0.03.
\end{flalign}
\end{subequations}
\end{lemma}

\begin{proof}
Define $Y = \frac{1}{\sigma \sqrt{n}} \sum_{i=1}^n S_i - \frac{q \sqrt{n}}{\sigma}$ where $\sigma^2 = q(1-q)$ is the variance of a Bernoulli random variable.
The Berry-Esseen theorem states that
\begin{flalign}\label{eq:berry-esseen}
\abs{\P(Y \le y) - \Phi(y)} \le \frac{C m_3}{\sigma^3 \sqrt{n}}
\end{flalign}
where $\Phi$ is the CDF for the normal distribution, $C= 0.5129$ \cite{korolev2010upper}, $\sigma$ is the standard deviation of $S_i$, and $m_3 = \E[ \abs{S_i - q}^3]$.
Since $n \ge 3000$ we get
\newcommand{\sigmaLB}{0.4999}
\[
\sigma^2 = q (1 - q) = \left( \frac{1}{2} - \frac{1}{16 \sqrt{n}} \right)  \left( \frac{1}{2} + \frac{1}{16 \sqrt{n}} \right) = \frac{1}{4} - \frac{1}{256 n} \ge \sigmaLB^2
\]
and
\[
m_3 = q (1-q)^3 + (1 - q) q^3 \le \max\{ q^3, (1-q)^3 \}\le \left( \frac{1}{2} + \frac{1}{16 \sqrt{n}} \right)^3 \le 0.13.
\] 
It follows that
\begin{flalign} 
\frac{C m_3}{\sigma^3 \sqrt{n}} \le \frac{0.5129 \times 0.13}{\sigmaLB^3 \sqrt{3000}} \le 0.01. \label{eq:const-berry-esseen}
\end{flalign}
Thus,
\begin{flalign*}
\P\left(  \sum_{i=1}^n S_i  > \frac{n}{2} + \frac{4}{7} \sqrt{2 n}  \right) &= \P\left(  \frac{1}{\sigma \sqrt{n}} \sum_{i=1}^n S_i - \frac{q \sqrt{n}}{\sigma} > \sqrt{n} \frac{\frac{1}{2} - q}{\sigma} + \frac{4  \sqrt{2}}{7 \sigma}  \right) \\
&= \P\left(  Y > \frac{1}{16 \sigma} + \frac{4 \sqrt{2}}{7 \sigma}   \right) \\
&\ge \P\left(  Y > \frac{1}{16 \times \sigmaLB} + \frac{4 \sqrt{2}}{7 \times \sigmaLB}   \right)  \\
&= 1  - \P\left(  Y \le\frac{1}{16 \times \sigmaLB} + \frac{4 \sqrt{2}}{7 \times \sigmaLB} \right)  \\
&\ge 1 - \Phi\left( \frac{1}{16 \times \sigmaLB} + \frac{4 \sqrt{2}}{7 \times \sigmaLB}\right) - \frac{C m_3}{\sigma^3 \sqrt{n}} \ge 0.04 -  0.01 = 0.03
\end{flalign*}
where the first inequality uses that $\sigma \ge \sigmaLB$, the second inequality uses \Cref{eq:berry-esseen}, and the last inequality uses \Cref{eq:const-berry-esseen}. Similarly,
\begin{flalign*}
	\P\left(  \sum_{i=1}^n S_i  > \frac{n}{2} \right) &= \P\left(  Y > \frac{1}{16 \sigma} \right) \ge \P\left(  Y > \frac{1}{16 \times \sigmaLB} \right) = 1 - \P\left( Y \le \frac{1}{16 \times \sigmaLB} \right) \\
	&\ge 1 - \Phi\left( \frac{1}{16 \times \sigmaLB}  \right) -  \frac{C m_3}{\sigma^3 \sqrt{n}}  \ge 0.45 -  0.01 \ge 0.44.
\end{flalign*}
\end{proof}

\begin{lemma}\label{lem:rho-has-chance-of-being-big}
If $n \ge 3000$  then  $\P \left( \frac{\eta_{\max}}{36} < \AdaSGD{\eta_{\max}} \right) \ge 0.003$.
\end{lemma}

\begin{proof}
Let $W = \sum_{i=n+1}^{2n} S_i$ and denote $\AdaSGD{\eta_{\max}}$ by $\bar{u}_n$.
By \citet[Theorem 4.14 and Theorem 3.1]{orabona2021modern} we have 
\[
\E\left[ \sum_{i=n+1}^{2n} f_i(\bar{u}_n)  -  f_i(\eta_{\max})  ~\Big|  W = w  \right] \le \eta_{\max} \sqrt{2 n}.
\]
Applying Markov's inequality yields
\[
\P\left(   \sum_{i=n+1}^{2n} f_i(\bar{u}_n)  -  f_i(\eta_{\max})  \ge \frac{10}{9} \eta_{\max} \sqrt{2 n} ~\Big| W = w \right) \le \frac{9}{10} 
\]
\[
\implies \P\left(   \sum_{i=n+1}^{2n} f_i(\bar{u}_n)  -  f_i(\eta_{\max})   < \frac{10}{9} \eta_{\max} \sqrt{2 n} ~\Big| W = w \right) \ge \frac{1}{10}.
\] 
Observe that
\begin{flalign*} 
	\sum_{i=n+1}^{2n} f_i(\bar{u}_n)  -  f_i(\eta_{\max})  &=  (n - w) \abs{\bar{u}_n} -  w \bar{u}_n  - (n - w) \eta_{\max} + w \eta_{\max} \ge (n - 2 w)  (\bar{u}_n - \eta_{\max}).
\end{flalign*}
Thus, if $2 w - n \ge \frac{8}{7} \sqrt{2 n} > 0$  then
\begin{flalign*}
	\frac{1}{10} &\le \P\left(   \sum_{i=n+1}^{2n} f_i(\bar{u}_n)  -  f_i(\eta_{\max}) < \frac{10 \eta_{\max} \sqrt{2 n}}{9}  ~\Big| W = w \right) \\
	&\le \P\left( (n - 2 w)  (\bar{u}_n - \eta_{\max})  < \frac{10 \eta_{\max} \sqrt{2 n}  }{9}  ~\Big| W = w  \right) \\
	&= \P\left( \bar{u}_n  \edit{>} \eta_{\max} \left( 1 - \frac{10 \sqrt{2 n}}{9 (2 w - n)} \right)   ~\Big| W = w  \right) \le \P\left( \edit{\bar{u}_n > \frac{\eta_{\max}}{36}}   ~\Big| W = w  \right).
\end{flalign*}
Thus, we get
\begin{flalign*}
\P\left(  \bar{u}_n > \frac{\eta_{\max}}{36} \right) &\ge \sum_{w = \lceil \frac{n}{2} + \frac{4}{7} \sqrt{2 n}  \rceil }^{n} \P\left(  \bar{u}_n > \frac{\eta_{\max}}{36} ~\Big| W = w  \right) \P( W = w) \\
&\ge \frac{1}{10} \sum_{w = \lceil \frac{n}{2} + \frac{4}{7} \sqrt{2 n}\rceil }^{n} \P( W = w)  \ge \frac{1}{10}\P\left( W > \frac{n}{2} + \frac{4}{7} \sqrt{2 n}  \right).
\end{flalign*}
Applying \Cref{eq:concentrate-binomial-distribution}  we get the desired result.
\end{proof}

\begin{proof}[Proof of \Cref{prop:greedy-model-selection-performs-poorly}]
	If  $x \le 0$ then
	\[
	\frac{1}{n} \sum_{i=1}^n f_i(x) = \frac{1}{n}  \left( \sum_{i=1}^n (1 - S_i) \abs{x} - \sum_{i=1}^n S_i x \right) = \frac{1}{n}  \left( \sum_{i=1}^n (S_i - 1) x - \sum_{i=1}^n S_i x \right)  = -x \ge 0.
	\]
	On the other hand, if $x \ge 0$ then
	\[
	\frac{1}{n} \sum_{i=1}^n f_i(x) = \frac{1}{n}  \left( \sum_{i=1}^n (1 - S_i) \abs{x} - \sum_{i=1}^n S_i x \right) = \frac{1}{n}  \left( \sum_{i=1}^n (1 - S_i) x - \sum_{i=1}^n S_i x \right)  = \frac{x}{n}  \left( n - 2 \sum_{i=1}^n S_i \right).
	\]
	Therefore, 
	\begin{flalign}\label{eq:when-greedy-is-max} 
	\sum_{i=1}^n S_i  >\frac{n}{2} \iff n - 2 \sum_{i=1}^n S_i < 0
	\implies x_{\kGreedy}  = \max_{k \in [K]} x_k.
\end{flalign}
Thus
	\begin{flalign*}
		\P \left( \frac{\eta_{\max}}{36} <  x_{\kGreedy} \right) &\ge \P \left( \frac{\eta_{\max}}{36} < \AdaSGD{\eta_{\max}},  x_{\kGreedy}  = \max_{k \in [K]} x_k \right) \\
		&  \us{\ge}{i} \P \left( \frac{\eta_{\max}}{36} < \AdaSGD{\eta_{\max}},  \sum_{i=1}^n S_i  >\frac{n}{2} \right)  \\
		&\us{\ge}{ii} \P \left( \frac{\eta_{\max}}{36} < \AdaSGD{\eta_{\max}} \right) \P\left( \sum_{i=1}^n S_i  >\frac{n}{2} \right) \us{\ge}{iii} 0.44 \times 0.003 \\
		&\ge \frac{1}{1000}.
	\end{flalign*}
	where $(i)$ uses \eqref{eq:when-greedy-is-max}, $(ii)$ uses that the training set and validation set are independent, and $(iii)$ uses \Cref{lem:validation-is-not-representitive} and \Cref{lem:rho-has-chance-of-being-big}. 
	Moreover, if $x_{\kGreedy}  > \frac{\eta_{\max}}{36}$ then since by definition $F(x) = (1 -q) \abs{x} -  q x $ and $q  = \frac{1}{2} - \frac{1}{16 \sqrt{n}}$ we get 
	\begin{flalign*}
		F(x_{\kGreedy} ) &= (1 -q) \abs{x_{\kGreedy}} -  q x_{\kGreedy}  \ge  (1 -q) x_{\kGreedy} - q x_{\kGreedy} = (1 - 2 q) x_{\kGreedy} \\
		&= \frac{x_{\kGreedy}}{8 \sqrt{n}} \ge \frac{\eta_{\max}}{288 \sqrt{n}} \ .
	\end{flalign*}
\end{proof}

\ShortPaper{
\subsection{Proof of \Cref{thm:reliable-guarantees}}\label{app:prove:prop:reliable-strongly-convex}

\ProofOfReliableProp{}

}

\mysection{Supplementary material for \Cref{sec:adaptivity-for-free}}{Supplementary Material for \Cref{sec:adaptivity-for-free}}
\newcommand{\vUBj}{C_j}

\edit{This appendix contains supplementary details for \Cref{sec:adaptivity-for-free}, particularly establishing the concentration inequalities in \Cref{lem:vector-concentration-bounds}, providing more details on \Cref{table:combinations-that-satisfy-general-adaptivity-for-free}, grid searching over $\lambda[]$, and providing the proof of \Cref{example:combine-methods} and the associated algorithm.}

\subsection{Proof of \Cref{lem:vector-concentration-bounds}}\label{app:concentration-inequalities}

\subsubsection{Proof of \Cref{lem:vector-concentration-bounds}.\ref{lem:vector-concentration-bounds:1a}}

\begin{lemma}\label{lem:l2-vector-concentration-inequality}
Let $V_1, \dots, V_n$ be a sequence of i.i.d. random vectors in $\R^d$ and let $\nu := \E[V_i]$. Let $H > 0$ be a constant. If $\| V_i - \nu \|_2 \le H$ almost surely, then for all $n \in \mathbb{N}$,
\[ 
\P\left( \left\| \frac{1}{n} \sum_{i=1}^n \left( V_i - \nu \right) \right\|_2 \ge \sqrt{\frac{2 \ln(2/\delta) \E \| V_i - \nu \|_2^2}{n} } + \frac{2 H \ln(2/\delta)}{3n} \right) \le \delta
\]
\end{lemma}

\begin{proof} 
This result follows from Corollary 10b of \cite{howard2020time} setting $Y_t = \frac{1}{n} \sum_{i=1}^{t} (V_i - \nu)$, $\Psi (\cdot) = \left\| \cdot \right\|_2$, and $c=H/n$. 
The selection of $\Psi(\cdot) = \left\| \cdot \right\|_2$ yields $D_\star=1$. Set $m = \E \| V_i - \nu \|_2^2 / n = \sum_{i=1}^{n} \E  \left\| \left(V_i - \nu\right)/n \right\|_2^2 \ge \sum_{i=1}^{t} \E \left\| \left( V_i - \nu \right) / n \right\|_2^2  = \sum_{i=1}^{t} \E_{i-1} \left\| \left( V_i - \nu \right) / n \right\|_2^2 = V_t$ for all $t \le n$, then $D_\star^2\mathfrak{s}_P \left( \frac{z}{m} \right) \cdot \left( V_t - m \right) \leq 0$. 
Therefore, for any $x$ we have
\[
\P \left(  \max_{t \le n}  \left\| \frac{1}{n} \sum_{i=1}^t \left( V_i - \nu \right) \right\|_2 \geq x \right) \leq 2 \exp \left(-\frac{x^2}{2m + (2/3n) Hx} \right).
\]
By \Cref{lem:prob_inequality_equivalent_statements} with $a=\sqrt{2m}$, $b=(2/3n) H$, and $c=2$, the following statement \edit{holds},
\[
\P \left( \max_{t \le n} \left\| \frac{1}{n} \sum_{i=1}^t \left( V_i - \nu \right) \right\|_2 \geq \sqrt{2m \ln(2/\delta)} + \frac{2 H \ln(2/\delta)}{3n} \right) \leq \delta.
\]
\end{proof}

\begin{lemma}\label{lem:empirical-bennett-wo-variance}
Let $Z, Z_1, \ldots, Z_n$ be i.i.d. random variables in $[0, H]$ almost surely where $H > 0$ is constant. Then with probability at least $1 - \delta$, $\E Z \le \frac{2}{n} \sum_{i=1}^{n} Z_i + \frac{13 H \ln(2 / \delta)}{3 (n-1)}$.
\end{lemma}

\begin{proof}
We can bound the sample variance using that $(i)$ $\bar{Z}$ minimizes the empirical sum of squares where $\bar{Z} = \frac{1}{n} \sum_{i=1}^n Z_i$ and $(ii)$ $Z^2 \le HZ$,
\[
\frac{1}{n-1} \sum_{i=1}^{n} \left( Z_i - \bar{Z} \right)^2 \us{\le}{i} \frac{1}{n-1} \sum_{i=1}^{n} Z_i^2 \us{\le}{ii} \frac{H}{n-1} \sum_{i=1}^{n} Z_i.
\]
Combining this with \Cref{thm:empirical-bennett} yields that with probability at least $1 - \delta$
\begin{flalign}\label{eq:concentration-bound-expect-Z-by-emperical}
\E Z \le \frac{1}{n} \sum_{i=1}^{n} Z_i + \sqrt{\frac{2 H \ln(2 / \delta) \sum_{i=1}^{n} Z_i}{n(n-1)}} + \frac{(7/3) H \ln(2/\delta)}{n-1}.
\end{flalign}
We will simplify \Cref{eq:concentration-bound-expect-Z-by-emperical} by analyzing two cases.
If $\frac{1}{n}\sum_{i=1}^{n} Z_i \le \frac{2H\ln(2/\delta)}{n-1}$, then \Cref{eq:concentration-bound-expect-Z-by-emperical} becomes
\[
\E Z \le \frac{1}{n} \sum_{i=1}^{n} Z_i + \frac{13 H \ln(2/\delta)}{3(n-1)}.
\]
If $\frac{1}{n}\sum_{i=1}^{n} Z_i > \frac{2H\ln(2/\delta)}{n-1}$, then \Cref{eq:concentration-bound-expect-Z-by-emperical} becomes
\[
\E Z < \frac{1}{n} \sum_{i=1}^{n} Z_i + \sqrt{\frac{1}{n^2} \left( \sum_{i=1}^{n} Z_i \right)^2 }  + \frac{(7/3) H \ln(2/\delta)}{n-1} = \frac{2}{n} \sum_{i=1}^{n} Z_i + \frac{(7/3) H \ln(2/\delta)}{n-1}.
\]
Combining these two cases yields our desired inequality.
\end{proof}

\begin{lemma}\label{lem:concentration-sample-variance}
Let $V_1, \dots, V_n$ be a sequence of i.i.d. random vectors in $\R^d$ and let $\nu := \E[V_i]$, and $\bar{V} := \frac{1}{n} \sum_{i=1}^n V_i$ for $n \ge 1$.
Let $H > 0$ be a constant.
If $\| V_i - \nu \|_2 \le H$ almost surely, then for all $n \in \mathbb{N}$, with probability $1 - \delta$, $\sum_{i=1}^{n} \| V_i - \nu \|_2^2 \le \sum_{i=1}^{n}  \| V_i - \bar{V} \|_2^2  + 2 H^2\ln(2/\delta)$.
\end{lemma}

\begin{proof} First, note that by adding and subtracting $\bar{V}$ and getting that the summation of the cross term is $0$, 
\begin{flalign}
\sum_{i=1}^{n} \| V_i - \nu \|_2^2 &= \sum_{i=1}^{n} \| V_i - \bar{V} + \bar{V} - \nu \|_2^2 = \sum_{i=1}^{n} (\| V_i - \bar{V} \|_2^2 +2 (V_i - \bar{V})^\top (\bar{V} - \nu ) +  \| \bar{V} - \nu \|_2^2)  \nonumber  \\
&= n \|\bar{V} - \nu \|_2^2 + \sum_{i=1}^{n} \| V_i - \bar{V}\|_2^2 + 2 (\bar{V} - \nu)  \left( -n\bar{V} + \sum_{i=1}^{n}  V_i \right)  \nonumber \\
&=  n \|\bar{V} - \nu \|_2^2 + \sum_{i=1}^{n} \| V_i - \bar{V}\|_2^2. \label{eq:V-i-minus-nu-norm-squared}
\end{flalign}
Then by Corollary 10a of \citet{howard2020time} with $Y_t = \frac{1}{n} \sum_{i=1}^{t} (V_i - \nu)$, $c_t = H/n$, $\Psi(\cdot)=\|\cdot\|_2$, $D=1$, $m = H^2/n$, $x=\sqrt{ (m/2) \ln 2/\delta}$, we get that with probability $1-\delta$ that
\begin{flalign}\label{eq:bar-V-minus-nu-determinstic-bounded} 
\| \bar{V} - \nu \|_2 \le H \sqrt{\frac{2\ln(2/\delta)}{n}}.
\end{flalign}
Combining \Cref{eq:bar-V-minus-nu-determinstic-bounded} and \Cref{eq:V-i-minus-nu-norm-squared} gives
\[
\sum_{i=1}^{n} \| V_i - \nu \|_2^2 = n \|\bar{V} - \nu \|_2^2 + \sum_{i=1}^{n} \| V_i - \bar{V}\|_2^2 \le 2 H^2\ln(2/\delta) + \sum_{i=1}^{n} \| V_i - \bar{V} \|_2^2 
\]
as desired.
\end{proof}

\begin{proof}[Proof of \Cref{lem:vector-concentration-bounds}.\ref{lem:vector-concentration-bounds:1a}]
Using \Cref{lem:l2-vector-concentration-inequality} with $H=2C$, as \Cref{lem:vector-concentration-bounds}.\ref{lem:vector-concentration-bounds:1a} assumes $\| V_i \|_2 \le C$ so  $\| V_i - \nu \|_2 \le 2C$, implies that with probability at most $\delta/3$ that
\[ 
 \left\| \bar{V} - \nu \right\|_2 \ge \sqrt{\frac{2 \ln(6/\delta) \E \| V_i - \nu \|_2^2}{n} } + \frac{4 C \ln(6/\delta)}{3 n}.
\]
By \Cref{lem:empirical-bennett-wo-variance}, with $Z_i = \| V_i - \nu \|_2^2$, and therefore $H=4C^2$, we get that with probability at most $\delta/3$ that
\[
\E\| V_i - \nu \|_2^2 > \frac{2}{n} \sum_{i=1}^{n} \| V_i - \nu \|_2^2 + \frac{52 C^2 \ln(6/\delta)}{3 (n-1)}.
\]
\Cref{lem:concentration-sample-variance} implies that with probability at most $\delta/3$ that
\[
\sum_{i=1}^{n} \| V_i - \nu \|_2^2 > \sum_{i=1}^{n}  \| V_i - \bar{V} \|_2^2  + 8 C^2\ln(6/\delta).
\]
Applying a union bound to these three statements, all three complementary inequalities hold simultaneously with probability at least $1-\delta$. Chaining these bounds together, and using the triangle inequality, we obtain with probability at least $1 - \delta$ that
\begin{flalign*}
 \left\| \bar{V} - \nu \right\|_2  &\le \sqrt{\frac{2 \ln(6/\delta) \E \| V_i - \nu \|_2^2}{n} } + \frac{4 C \ln(6/\delta)}{3 n} \\
 &\le  \sqrt{\frac{2 \ln(6/\delta)}{n} \cdot \left(  \frac{2}{n} \sum_{i=1}^{n} \| V_i - \nu \|_2^2 + \frac{52 C^2 \ln(6/\delta)}{3 (n-1)} \right)} + \frac{4 C \ln(6/\delta)}{3 n} \\
 &\le  \sqrt{\frac{2 \ln(6/\delta)}{n} \cdot \left(  \frac{2}{n} \left( \sum_{i=1}^{n}  \| V_i - \bar{V} \|_2^2  + 8 C^2\ln(6/\delta) \right) + \frac{52 C^2 \ln(6/\delta)}{3 (n-1)} \right)} + \frac{4 C \ln(6/\delta)}{3 n} \\
  &\le \frac{2 \sqrt{ \ln(6/\delta)\sum_{i=1}^{n}  \| V_i - \bar{V} \|_2^2  }}{n}  + \frac{10 C \ln(6/\delta)}{n-1}
\end{flalign*}
as desired.
\end{proof} 

\subsubsection{Proof of \Cref{lem:vector-concentration-bounds}.\ref{lem:vector-concentration-bounds:2b}}

\begin{proof} 
The result follows by \Cref{thm:empirical-bennett} with $a = -\Vub_j$, $b = \Vub_j$, $Z_i = \coor{V_i}$ for each $j \in [d]$,
 and then applying a union bound.
\end{proof} 

\subsubsection{Proof of \Cref{lem:vector-concentration-bounds}.\ref{lem:vector-concentration-bounds:2d}}

\begin{proof}
For each coordinate $j \in [d]$ let the sample variance be 
	$s_j^2 := \frac{1}{n-1} \sum_{i=1}^n \coor{V_i - \bar{V}}^2$ and the population variance be $\sigma_j^2 := \frac{1}{n} \sum_{i=1}^n \coor{V_i - \E[V_i]}^2$.
	 \Cref{thm:bennett-inequality} with $a = -\Vub_j$, $b = \Vub_j$, $Z_i = \coor{V_i}$ for each $j \in [d]$ implies that 
\[
\P\left( \abs{ \coor{\nu - \bar{V}_n} } > \sigma_j \sqrt{\frac{2 \ln 2 /\delta}{n}} + \frac{2 \Vub_j \ln 2/\delta}{3 n}  \right) \le \delta.
\]
Applying \Cref{lem:dependency-sum-lemma} with $X_j =  \abs{ \coor{\nu - \bar{V}_n} }$  to this inequality implies that 
\begin{flalign}\label{eq:L1-error-sum-pop-variance}
	\P\left( \sum_{j=1}^d \abs{ \coor{\nu - \bar{V}_n} } > \frac{9}{4} \sum_{j=1}^d \sigma_j \sqrt{\frac{2 \ln 12 /\delta}{n}} + \frac{\edit{4} \Vub_j \ln 12/\delta}{3 n}  \right) \le \delta.
\end{flalign}
Equation~(3) of \citet{maurer2009empirical} implies that
\[
\P\left( \sigma_j - s_j > 2  \vUBj  \sqrt{\frac{2 \ln(1/\delta)}{n-1}}  \right) \le \delta.
\]
This inequality implies that, for some fixed $\varepsilon > 0$,
\[
\P\left( \sigma_j - s_j - \varepsilon \ge 2  \vUBj  \sqrt{\frac{2 \ln(1/\delta)}{n-1}}  \right) \le \delta.
\]
Applying \Cref{lem:dependency-sum-lemma} to this inequality with $X_j = \sigma_j - s_j - \varepsilon$ implies that 
\begin{flalign}\label{eq:difference-empirical-and-pop-variance}
\P\left( \sum_{j=1}^d \sigma_j - s_j > \frac{9}{2}  \sqrt{\frac{2 \ln(6/\delta)}{n-1}}  \sum_{j=1}^d  \vUBj \right) \le \delta. 
\end{flalign}
Taking a union bound over \Cref{eq:L1-error-sum-pop-variance} with $\frac{2\delta}{3}$ and \Cref{eq:difference-empirical-and-pop-variance} with $\frac{\delta}{3}$ gives the result.
\end{proof}

\ShortPaper{
\subsection{Proof of \Cref{lem:bound-regularization-direction}}\label{app:proof-bound-regularization-direction}

\ProofOfBoundedNorm

}

\subsection{Explanation of \Cref{table:combinations-that-satisfy-general-adaptivity-for-free}}\label{app:coro:applications-of-main-theorem}

Note that the $\lambda[]$ values are found by substituting $V_i = \grad f_i(x)$ into \Cref{lem:vector-concentration-bounds}
and employing Assumption~\ref{assume:Lip-estimates}.

{
\sloppy
The remainder of this section defines  $\AdaSGD{R\edit{; f_{n+1}, \ldots, f_{2n}}}$, $\SimplexMirrorDescent{R\edit{; f_{n+1}, \ldots, f_{2n}}}$ and $\AdaGrad{R\edit{; f_{n+1}, \ldots, f_{2n}}}$, and recaps their known complexities.
In particular, $\AdaSGD{R\edit{; f_{n+1}, \ldots, f_{2n}}} = \bar{u}_n$ where $\bar{u}_n$ is computed from \Cref{eq:adaptive-SGD} starting from $u_1 = \zeros$.  \citet[Theorem 3.9 \& 4.14]{orabona2021modern} prove that with probability at least $1-\delta$ \par
}
\begin{flalign}\label{eq:adaptive-SGD-suboptimality-bound}
F( \AdaSGD{R\edit{; f_{n+1}, \ldots, f_{2n}}}  ) - \min_{ x \in \X : \| x \|_2 \le R } F(x) \le O\left(  \frac{\Lip R}{\sqrt{n}} \sqrt{\ln 1/\delta } \right).
\end{flalign}

Next, we consider $\SimplexMirrorDescent{R\edit{; f_{n+1}, \ldots, f_{2n}}}$. Without loss of generality assume $\X \subseteq \{ x \in \R^{d} : x \ge \zeros \}$, since we can always decompose $x = x_{+} - x_{-}$ where each component of $x_{+}$ and $x_{-}$ is greater than or equal to zero. The update equations for mirror descent starting from $u_1=\zeros$ for $t=1,\dots,n$ are
\begin{flalign*}
\bar{u}_t &\gets \frac{1}{t} \sum_{l=1}^t u_l \\
u_{t+1} &\in \argmin_{u \in \mathcal{X} : \| u \|_1 \le R} \left\{\eta_t \grad f_{n+t}(u_t)^\top u + D_{\Phi}(u, u_t)\right\} \\
\eta_t &\gets \frac{R \sqrt{2}}{2 \sqrt{\sum_{l=1}^t \| \grad f_{n+l}(u_l) \|_{\infty}^2}}
\end{flalign*}
where $\Phi(x) := \sum_{j=1}^d \coor{x} \ln \coor{x}$ and $D_{\Phi}$ is the Bregman divergence \cite{beck2003mirror}.
Setting $\SimplexMirrorDescent{R\edit{; f_{n+1}, \ldots, f_{2n}}} = \bar{u}_n$, gives \cite[Theorem 6.11 \& 3.14]{orabona2021modern} that with probability $1-\delta$,
\begin{flalign}\label{eq:mirror-simplex-suboptimality-bound}
F(\SimplexMirrorDescent{R\edit{; f_{n+1}, \ldots, f_{2n}}}) - \min_{x \in \X : \| x \|_{1} \le R} F(x) \le O\left(  \frac{\| \LipCoord \|_{\infty} R}{\sqrt{n}} \sqrt{\ln d/\delta } \right).
\end{flalign}
Finally, let
$\AdaGrad{R\edit{; f_{n+1}, \ldots, f_{2n}}} = \bar{u}_{n}$ where for $t = 1, \dots, n$,
\begin{flalign*}
\bar{u}_t &\gets \frac{1}{t} \sum_{i=1}^{t} u_i \\
 u_{t+1} &\gets \argmin_{u \in \X : \| u \|_{\infty} \le R} \| u_{t} - R G_t^{-1/2} \grad f_{n+t}(u_t) - u \|_{G_t^{1/2}} \\ 
 G_{t+1} &\gets G_{t} + \mathbf{diag}(\grad f_{n+t}(u_t))^2
\end{flalign*}
$G_{0} = \zeros$, and $\| v \|_{H} := \sqrt{v^\top H v}$.
From \cite[Section 3.2]{gupta2017unified} and \cite[Theorem 3.14]{orabona2021modern}, with probability $1-\delta$ this obtains the suboptimality guarantee: %
\begin{flalign}\label{eq:adagrad-suboptimality-bound} 
F(\AdaGrad{R\edit{; f_{n+1}, \ldots, f_{2n}}}) - \min_{x \in \X : \| x \|_{\infty} \le R } F(x) \le O\left(  \frac{\| \LipCoord \|_{1} R}{\sqrt{n}} \sqrt{\ln 1/\delta } \right).
\end{flalign}

\mysubsection{Grid searching over $\lambda[]$ to eliminate the intractable calculation in \Cref{table:combinations-that-satisfy-general-adaptivity-for-free}}{Grid Searching over $\lambda[]$ to Eliminate the Intractable Calculation in \Cref{table:combinations-that-satisfy-general-adaptivity-for-free}}\label{app:dealing-with-unknown-lambda}

This section explains how we can grid search over $\lambda[]$ using our \RMS{} method to avoid the intractable calculations for $\lambda[]$ currently present in \Cref{table:combinations-that-satisfy-general-adaptivity-for-free}.

Define,
\[
\ell_{p} := \begin{cases}
\hat{\Lip} & \text{ if } p = 2 \\
\| \hat{\LipCoord} \|_{\infty} & \text{ if } p = 1 \\
\| \hat{\LipCoord} \|_{1} & \text{ if } p = \infty,
\end{cases}
\qquad
l_{p} := \begin{cases}
\Lip & \text{ if } p = 2 \\
\| \LipCoord \|_{\infty} & \text{ if } p = 1 \\
\| \LipCoord \|_{1} & \text{ if } p = \infty,
\end{cases}
\]

\begin{flalign*}
\hatlambda[p]^{(0)} := \begin{cases}
\frac{4 \sqrt{\ln \frac{6}{\delta}}}{\sqrt{n}} + \frac{20 \ln \frac{6}{\delta}}{n-1} & \text{ if } p = 2 \\
\frac{\edit{2} \sqrt{2 \ln \frac{4 d}{\delta} }}{\sqrt{n-1}} +\frac{28 \ln \frac{4 d}{\delta}}{3 (n-1)} & \text{ if } p = 1 \\
\frac{9 \sqrt{2 \ln \frac{18}{\delta}}}{2 \sqrt{n-1}} +\frac{\edit{48} \ln \frac{18}{\delta}}{n-1} & \text{ if } p = \infty,
\end{cases}
\end{flalign*}
and
\[
\hatlambda[p]^{(k)} := \e^k \hatlambda[p]^{(0)} \text{ for } k=1, \dots, K_p ~~\text{where}~~ K_p := 1 \vee \lceil \ln  \ell_p  \rceil. 
\]
Let $\hat{x}_{k,p}$ be the output of \Cref{alg:PUD} with norm and algorithm specified in \Cref{table:combinations-that-satisfy-general-adaptivity-for-free} but regularizer $\hatlambda[p]^{(k)}$ instead of $\lambda[p]$, then due to \Cref{thm:instantiates-lambda-phi-table} we get with probability $1-3 \delta$ that for all $\hatlambda[p]^{(k)} \ge \lambda[p]$ that
\[
F(\hat{x}_{k,p}) - \Fstar \le (\edit{14} \lambdaAlg[p] + 3 \hatlambda[p]^{(k)} ) \DStar ~~\text{ and }~~ \| \hat{x}_{k,p} \| \le 14 \DStar.
\]
Applying this inequality for $\hatlambda[p]^{(k)} \in [\lambda[p], \e \lambda[p]]$ with 
\begin{flalign*}
\tau_{k,p} &= \sqrt{\frac{2 \tilde{c}_p \SampleVar{k}}{n}} + \tilde{c}_p \frac{14 \ell_p \| x_k \|}{3 (n-1)} \\
\tilde{c}_p &:= \ln \frac{4 K_p}{\delta}.
\end{flalign*}
gives by \Cref{eq:F-rely-x-k-guarantee} and a union bound that, with probability $1-4\delta$,
\[
F(x_{\kReliable,p}) - \Fstar \le  (\edit{14} \lambdaAlg[p] + \edit{3 \e} \lambda[p] ) \DStar + 14 (1 + \gamma) \left( \sqrt{2 \tilde{c}_p} \frac{l_{p} \DStar}{\sqrt{n\edit{-1}}} + 
 \tilde{c}_p \frac{14\ell_p \DStar}{3(n\edit{-1})} \right).
\]
Substituting the values for each of these terms and fixing $\gamma$ to be a constant gives
\[
F(x_{\kReliable,p}) - \Fstar = \begin{cases}
O\!\left( \DStar l_p \sqrt{\frac{\ln(\delta^{-1} \ln \ell_p)}{n-1}} + \frac{\ell_p \DStar \ln(\delta^{-1} \ln \ell_p)}{n-1} \right) & \text{if } p \in \{2, \infty\} \\[6pt]
O\!\left( \DStar l_1 \sqrt{\frac{\ln(d \, \delta^{-1} \ln \ell_1)}{n-1}} + \frac{\ell_1 \DStar \ln(d \, \delta^{-1} \ln \ell_1)}{n-1} \right) & \text{if } p = 1.
\end{cases}
\]
as desired. 
Note this approach uses $K_p$ calls to an empirical risk minimization oracle and the appropriate algorithm from \Cref{table:combinations-that-satisfy-general-adaptivity-for-free}.

\subsection{Proof of \Cref{example:combine-methods}} \label{app:example:combine-methods}

\begin{algorithm}
\caption{Simultaneously adapting to multiple problem structures}
\label{alg:simultaneously-adapting}
\begin{algorithmic}
\INPUT Sample functions 
		$f_1, \dots, f_{3n}$ 
\INPUT Maximum failure probability $\delta > 0$
\INPUT Lipschitz estimates $\LipHat$ and $\hat{\LipCoord}$

\FOR{$k = 1, 2, 3$}
    \STATE Compute $x_k = \alg(2 \| \xReg{\lambda[]} \|\edit{; f_{n+1}, \dots, f_{2n}} )$ using \edit{algorithm} $\alg$, \edit{norm} $\| \cdot \|$ and \edit{regularizer} $\lambda[]$ from $k$th row of Table~\ref{table:combinations-that-satisfy-general-adaptivity-for-free} where $\xReg{\lambda[]} \in \argmin_{x \in \X} \bar{F}(x) + \lambda[] \| x \|$.
    \STATE $\SampleVar{k} = \frac{1}{n-1} \sum_{i=1}^n (f_i(x_k) - f_i(\zeros) - (\barF{k}  - \bar{F}(\zeros)))^2$    
    \STATE $\tau_k \gets \sqrt{\frac{2 \SampleVar{k}}{n} \ln \frac{12}{\delta} } + \frac{14}{3} \cdot \frac{\ln \frac{12}{\delta}}{n-1} (\LipHat \| x_k \|_2  \wedge \| \hat{\LipCoord} \|_{\infty} \| x_k \|_1 \wedge \| \hat{\LipCoord} \|_1 \| x_k \|_{\infty} )$
\ENDFOR

\OUTPUT $z \gets x_{\kReliable}$ from applying \RMS{} to $\{(x_k, \tau_k)\}_{k=1}^3$ on $f_{2n+1}, \dots, f_{3n}$ with fixed $\gamma \in [1,\infty)$.
\end{algorithmic}
\end{algorithm}

\begin{proof} 
We use \Cref{alg:simultaneously-adapting} to obtain the desired result.

Let $Z_{k,i} := F(x_k) - f_i(x_k) - (F(\zeros)  - f_i(\zeros))$ and $q_k := \hat{\Lip} \| x_k \|_2  \wedge \| \hat{\LipCoord} \|_{\infty} \| x_k \|_1 \wedge \| \hat{\LipCoord} \|_1 \| x_k \|_{\infty}$. Note that almost surely $\abs{Z_{k,i}} \le q_k$. 
Thus, by \Cref{thm:empirical-bennett} and a union bound we deduce that \Cref{assume:confidence-intervals} holds with the $\tau_k$ value specified in \Cref{alg:simultaneously-adapting}. 
Thus, \Cref{lem:reliable-selection-general-guarantee} implies that with probability $1-\delta$
\begin{flalign}\label{eq:tau-sub-opt}
F(x_{\kReliable}) - \Fstar \le \min_{k \in [K]} F(x_{k}) - \Fstar  + (1 + \gamma) \tau_k.
\end{flalign}
Moreover, by definition of $\tau_k$ and the assumption that $\Lip \le \hat{\Lip} \le L \sqrt{n / \ln(1/\delta)}  \implies \ln(1/\delta) \le n$ we get
\begin{flalign*}
	\tau_k &= \sqrt{\frac{2 \SampleVar{k}}{n} \ln \frac{12}{\delta} } + \frac{14}{3} \cdot \frac{\ln \frac{12}{\delta}}{n-1} q_k \le \sqrt{\frac{2 \SampleVar{k}}{n} \ln \frac{12}{\delta} } + \frac{14}{3} \cdot \sqrt{\frac{n}{n-1}} \cdot \sqrt{\frac{\ln(12) \ln \frac{12}{\delta}}{n-1}} q_k. 
\end{flalign*}

By Assumption~\ref{assume:Lip-estimates} and from \Cref{lem:bound-regularization-direction} \edit{applied separately for each $k \in \{1,2,3\}$ and combined via a union bound,} with probability $1-\edit{6}\delta$ we have $\| x_1 \|_2 \le 14 \| \xstar \|_2$, $\| x_2 \|_{1} \le 14 \| \xstar \|_1$ and $\| x_3 \|_{\infty} \le 14 \| \xstar \|_{\infty}$ for all $\xstar \in \XStar$. Thus, for all $\xstar \in \XStar$ we get with probability $1-\edit{6}\delta$ that
\begin{subequations}\label{eq:tau-ubs}
\begin{flalign}
\tau_1 &\le O\left( \frac{\Lip \| \xstar  \|_2 }{\sqrt{n}} \sqrt{\ln 1/\delta} \right) \\
 \tau_2 &\le O\left( \frac{\| \LipCoord \|_{\infty} \| \xstar \|_{1} }{\sqrt{n}} \sqrt{\ln 1/\delta} \right) \\
 \tau_3 &\le O\left( \frac{\|\LipCoord \|_{1} \| \xstar \|_{\infty} }{\sqrt{n}} \sqrt{\ln 1/\delta} \right) .
\end{flalign}
\end{subequations}
\Cref{thm:instantiates-lambda-phi-table}, a union bound, $\LipHat \le \Lip  \sqrt{n / \ln(1/\delta)}$, and $\hat{\LipCoord}_j \le \LipCoord_j \sqrt{n / \ln(d/\delta)}$ implies that, with probability $1-\edit{9}\delta$, for all $\xstar \in \XStar$
\begin{flalign*}
F(x_1) -  \Fstar &\le O\left( \frac{L \| \xstar  \|_2 }{\sqrt{n}} \sqrt{\ln 1/\delta} \right) \\
F(x_2) -  \Fstar &\le O\left( \frac{\| \LipCoord \|_{\infty} \| \xstar \|_{1} }{\sqrt{n}} \sqrt{\ln d/\delta} \right) \\
F(x_3) -  \Fstar &\le O\left( \frac{\| \LipCoord \|_{1} \| \xstar \|_{\infty} }{\sqrt{n}} \sqrt{\ln 1/\delta} \right) 
\end{flalign*}
Combining these inequalities with \Cref{eq:tau-sub-opt} and \Cref{eq:tau-ubs} using a union bound gives the result \edit{with probability $1 - 10\delta$}.
\end{proof}

\mysection{Justifying the norm for the experiments}{Justifying the Norm for the Experiments}\label{sec:justify-norm}

\newcommand{\softmaxvec}[1]{\mathbf{p}(#1)}
\newcommand{\softmax}[2]{\mathbf{p}_{#1}(#2)}

In this section, we show that for multivariate logistic regression problems with normalized features (as is the case for zero-shot CLIP models \cite[Section 3.1.2]{radford2021learning}), every sample $f(\cdot; S)$ is 2-Lipschitz with respect to the $\| \cdot \|_{2,\infty}$ norm (\Cref{eq:Lip-infty-2}). 

Define the softmax function as
\[
\softmax{l}{z} := \frac{\e^{\coor[l]{z}}}{\sum_{c=1}^C \e^{\coor[c]{z}} }  \quad \softmaxvec{z} := \begin{pmatrix}
	\softmax{1}{z} \\
	\vdots \\
	\softmax{C}{z}
\end{pmatrix}
\]
where $C$ is the number of classes and $z$ is a vector of scores.
In this section, with slight abuse of notation, we will let $S$ be a random vector of size $m+1$ with $S_0$ denoting the first element and $S_{1:m}$ the remaining elements.
For a matrix $D$, let $D_c$ be its $c$th row.
Let $\mathbf{e}_c$ be the vector with a one in the $c$th position and zeros everywhere else.

\begin{proposition}\label{prop:norm-choice-experiments}
Let $C$ be the number of classes in the classification problem and $m$ the number of features in a multivariate logistic regression problem with $\X = \R^{C \times m}$,
\[
f(X; S) = -\log(\softmax{S_0}{X S_{1:m}} ),  \quad \SS = \{ S \in \R^{m+1} : S_0 \in \{1,\dots, C \} \text{ and } \| S_{1:m} \|_2 \le 1 \}.
\]
Then,
\begin{flalign}\label{eq:Lip-infty-2}
\sup_{S \in \SS} \abs{f(X;S) - f(X';S)} \le 2 \| X - X' \|_{2,\infty} \quad \forall X, X' \in \R^{C \times m}.
\end{flalign}
\end{proposition}

\begin{proof}
Since $\grad f(X;S) = (\softmaxvec{X S_{1:m}} - \mathbf{e}_{S_0} ) S_{1:m}^\top$
we get that
\[
\| \grad f(X; S)  \|_{2,1} = \| (\softmaxvec{X S_{1:m}} - \mathbf{e}_{S_0} ) S_{1:m}^\top \|_{2,1}  \le \| \softmaxvec{X S_{1:m}} - \mathbf{e}_{S_0} \|_1 \| S_{1:m} \|_2 \le 2.
\]
Since $\| \cdot \|_{2,\infty}$ is dual to $\| \cdot \|_{2,1}$,
we deduce $f(\cdot ; S)$ is $2$-Lipschitz with respect to the $\| \cdot \|_{2, \infty}$ norm.
\end{proof}

\mysection{Further details of experiments}{Further Details of Experiments}\label{sec:further-experiment-details}

\myparagraph{Costs of running experiments}
We estimate the total cost of running our experiments for this paper (including unused and failed experiments), at current cloud computing prices, to be well under \$1,000  including purchasing tokens for Gemini.

\myparagraph{Assets used}
For the few-shot learning experiments with CIFAR-10 \cite{krizhevsky2009learning} and CLIP we used the
pandas \cite{mckinney2011pandas}, PyTorch \cite{paszke2019pytorch}, and OpenCLIP \cite{ilharco_gabriel_2021_5143773} packages
with the ViT-L/14-quickgelu model \cite{dosovitskiy2021image,radford2021learning} pretrained on the DFN-2B data set \cite{fang2024data}. 
In the shape-counting experiment, for image generation, prompting \texttt{Gemini-1.5-Flash}, and processing the results, we used the matplotlib \cite{matplotlibHunter}, google.generativeai \cite{team2024gemini}, pandas \cite{mckinney2011pandas}, and word2number \cite{word2number} packages;
for running the model selection methods, we used tidyverse \cite{tidyverse}. 

\myparagraph{Additional plots}
We provide some additional plots to supplement \Cref{fig:combined-experiment-plots} in the body of the paper.
\Cref{fig:additional-plots-for-CLIP} provides two additional plots for the few-shot learning experiments.
The left plot in \Cref{fig:additional-plots-for-CLIP} shows that the performance gap between
reliable and standard model selection in the left plot of \Cref{fig:combined-experiment-plots} is primarily
driven by poor tail performance. The right plot in \Cref{fig:additional-plots-for-CLIP} shows that reliable 
model selection is also reasonable at obtaining a good error rate, even though it is applied to minimizing cross-entropy loss.
\Cref{fig:additional-plot-for-shape} shows the 95th percentile for the prompt engineering experiments.
This shows a similar pattern to the right plot of \Cref{fig:combined-experiment-plots}. 

\begin{figure}[htb]
\includegraphics[width=0.5\textwidth]{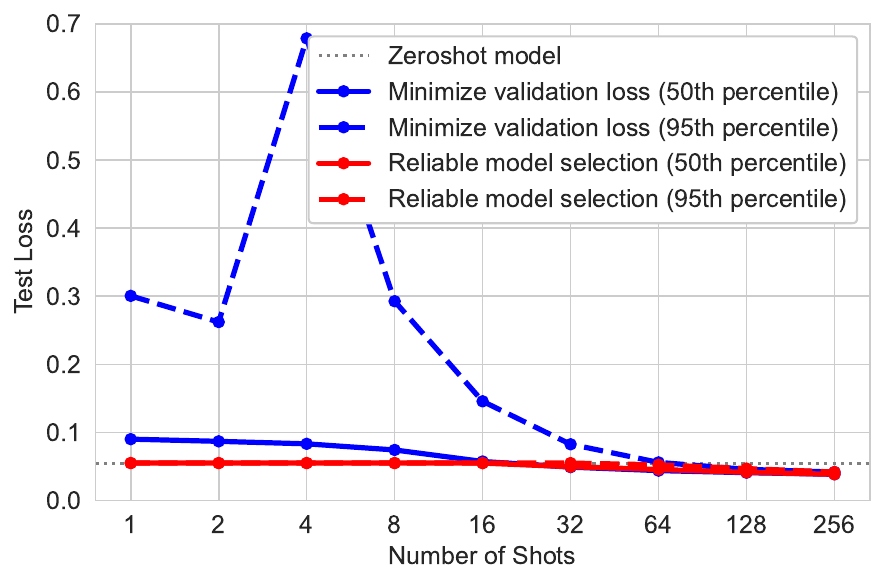}
\includegraphics[width=0.5\textwidth]{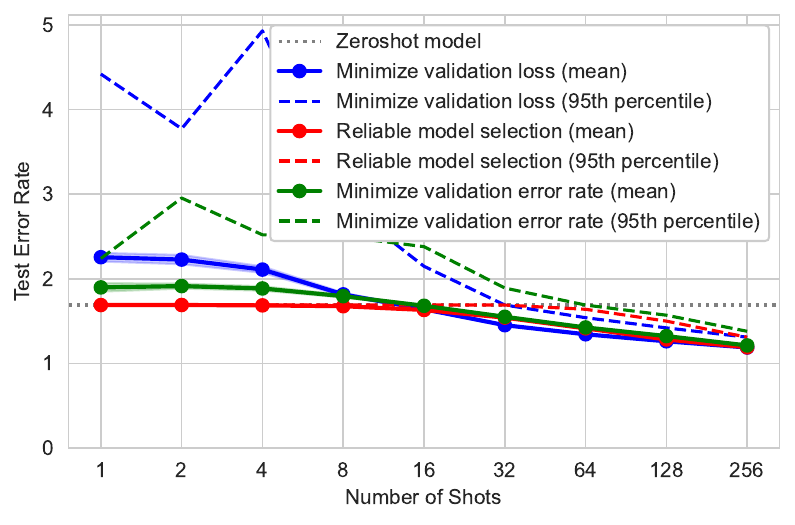}
\caption{Additional plots for the CLIP experiments. The percentiles are over the runs.
For the right plot with error rates, reliable model selection is applied to the validation loss (i.e., cross-entropy).}
\label{fig:additional-plots-for-CLIP}
\end{figure}

\begin{figure}[htb]
\centering
\includegraphics[width=0.7\textwidth]{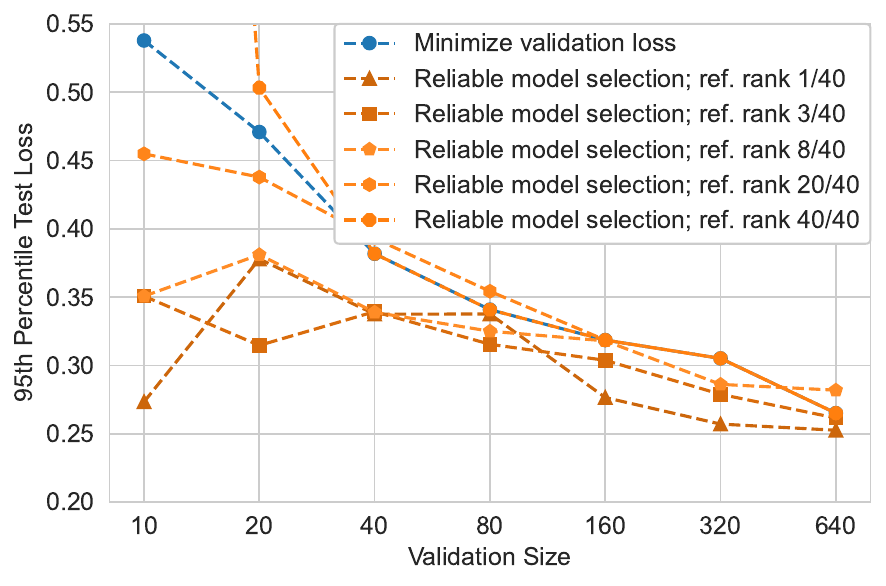}
\caption{Additional plot for the prompt engineering experiments. The percentiles are over the runs.}
\label{fig:additional-plot-for-shape}
\end{figure}

\myparagraph{More details for prompt engineering experiment}
Shape-counting images were systematically generated with the following criteria:
\begin{itemize}
    \item Each image contains between 1 and 12 shapes, sampled uniformly.
    \item Each shape belongs to one of six categories, sampled uniformly (ellipse, rectangle, triangle, hexagon, star, pentagon).
    \item Each shape is defined by a bounding box whose height and width are independently sampled uniformly between 10\% and 30\% of the image width.
    \item Shape coordinates are sampled uniformly at random, skipping placements that would cause bounding boxes to overlap or extend beyond the image boundaries.
    \item Each shape is filled with one of eight colors, sampled uniformly (red, blue, green, purple, orange, yellow, pink, cyan).
    \item The background of each image is sampled uniformly from 15 preset options: a children's ball pit, a tropical beach, blurry lights, a building, purple wavy lines, a dog, tulips, wavy glowing triangles, fading horizontal lines, distorted colorblocks, a Persian rug, a telescope's image of distant galaxies, a black and white spiral, van Gogh's \textit{Starry Night}, and an image of static white noise.
    \item Each image is square with background images rescaled to fit.
\end{itemize}

\begin{figure}[htbp]
    \centering
    \begin{tabular}{ccccc}
      \includegraphics[width=0.18\textwidth]{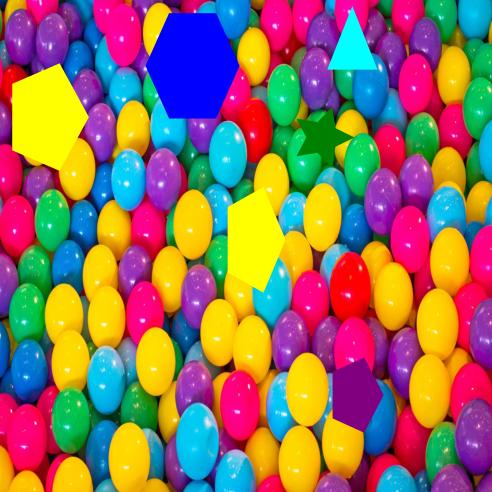} &
      \includegraphics[width=0.18\textwidth]{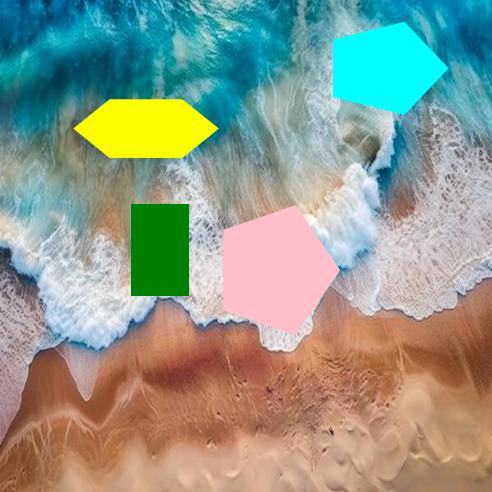} &
      \includegraphics[width=0.18\textwidth]{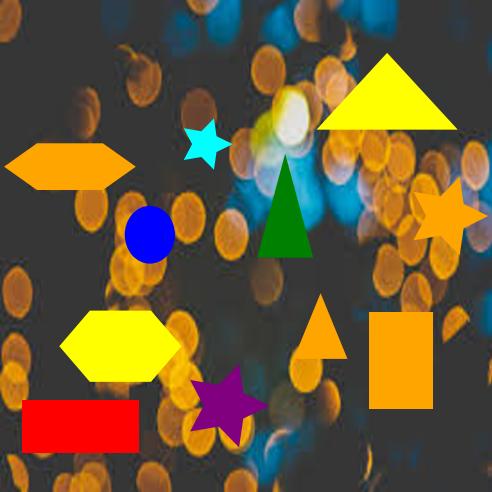} &
      \includegraphics[width=0.18\textwidth]{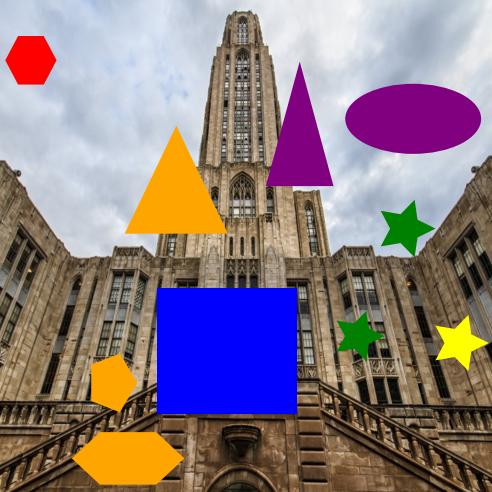} &
      \includegraphics[width=0.18\textwidth]{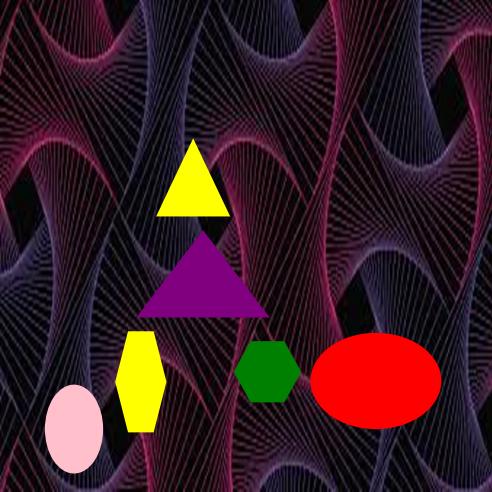} \\[4pt]
      \includegraphics[width=0.18\textwidth]{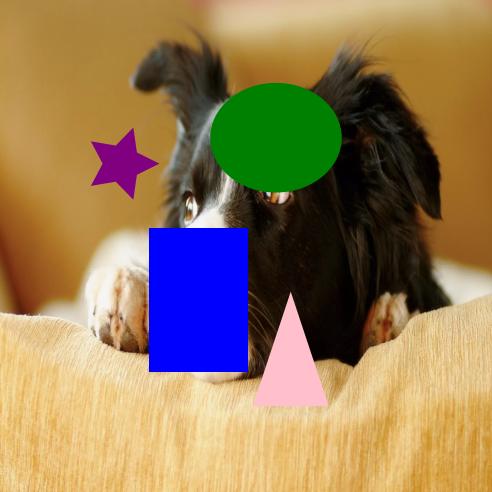} &
      \includegraphics[width=0.18\textwidth]{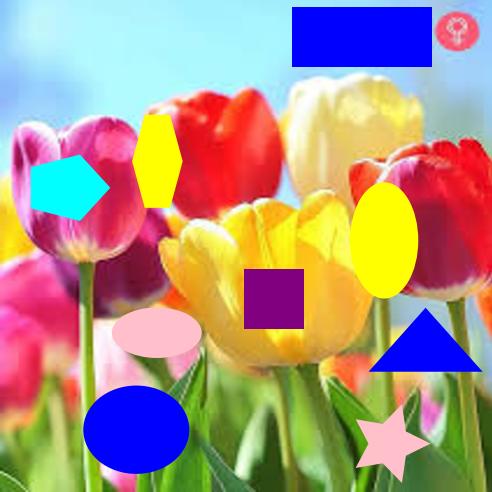} &
      \includegraphics[width=0.18\textwidth]{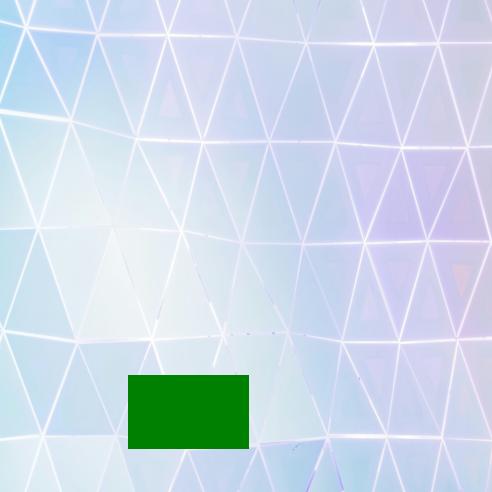} &
      \includegraphics[width=0.18\textwidth]{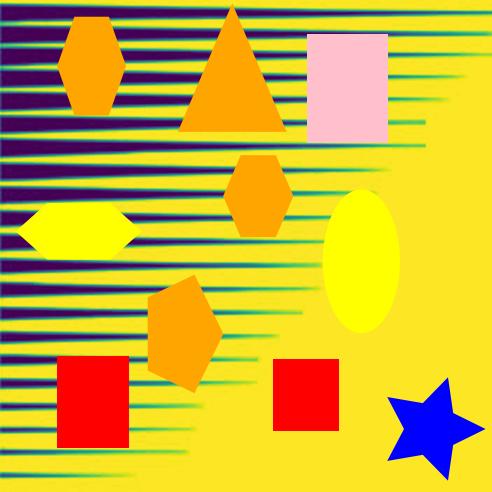} &
      \includegraphics[width=0.18\textwidth]{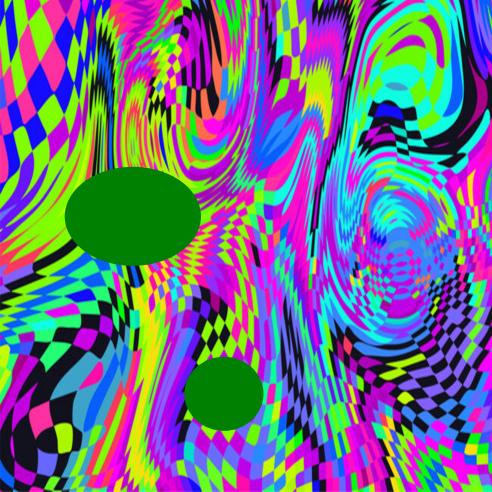} \\[4pt]
      \includegraphics[width=0.18\textwidth]{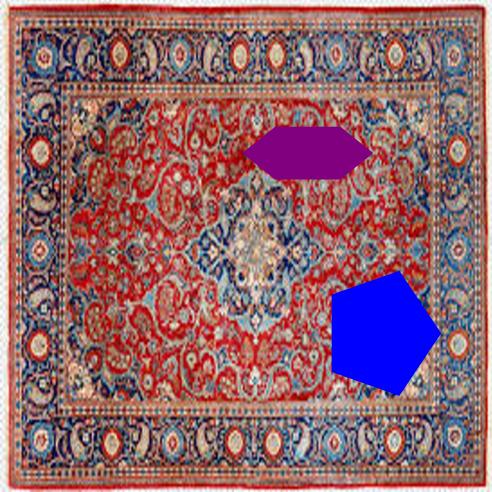} &
      \includegraphics[width=0.18\textwidth]{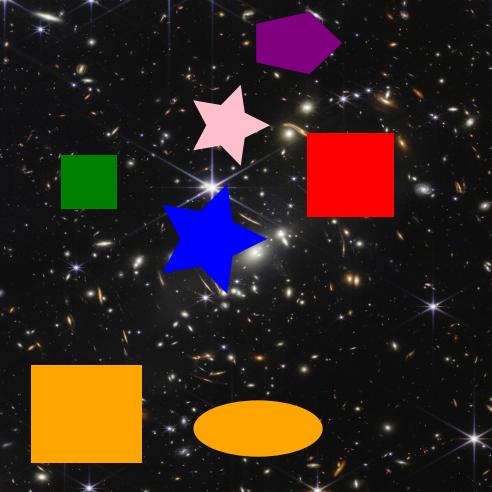} &
      \includegraphics[width=0.18\textwidth]{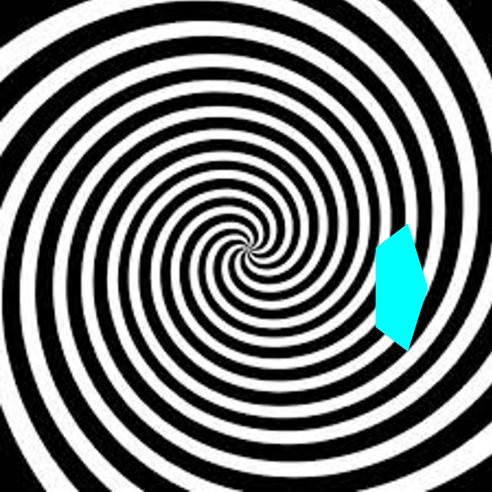} &
      \includegraphics[width=0.18\textwidth]{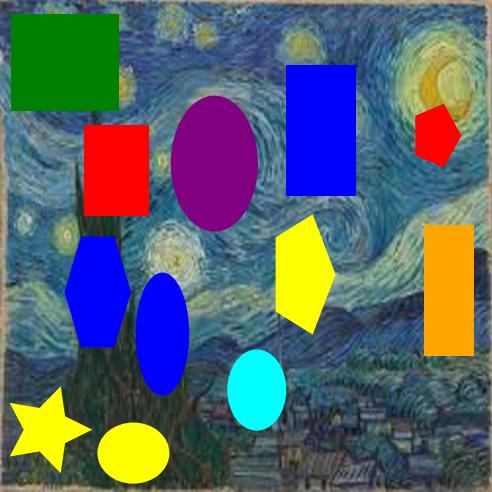} &
      \includegraphics[width=0.18\textwidth]{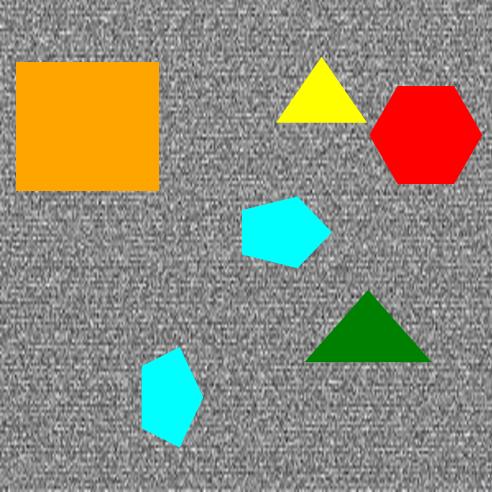}
    \end{tabular}
    \caption{A sample of images, generated by the outlined process, that were presented to Gemini showcasing all 15 backgrounds.}
    \label{fig:shapecountgrid}
  \end{figure}

We evaluate 40 distinct textual prompts on each generated image using \texttt{Gemini-1.5-Flash} at temperature 0. All images are sent with the same square sizing at 64 DPI resolution. Numerical responses are extracted from the LLM's output as follows:
\begin{enumerate}
 \item Remove all asterisks (\texttt{*}) left as formatting artifacts.
 \item Return the last token directly if it is purely numeric.
 \item Isolate content within braces \{\,\dots\,\} and within square brackets [\,\dots \,], returning 0 if a hyphenated range is detected.
 \item Perform a reverse search through each remaining word (what we have isolated in the previous step, or the entire output if an isolation was not made), first checking for numeric types and then attempting to extract numbers in word form. The first successful conversion is returned, or 0 otherwise.
 \item Clip all extracted values outside the range [0,\,12] to the nearest endpoint.
\end{enumerate}
This numerical extraction is compared to the true number of shapes in each image to determine the error, in this case absolute error.

\pgfplotsset{compat=1.18}
\newcolumntype{L}{>{\raggedright\arraybackslash}X}
\newcolumntype{C}{>{\centering\arraybackslash}X}
\renewcommand{\arraystretch}{1.2}

{\setlength{\LTcapwidth}{\linewidth}
\begin{longtable}{>{\raggedright\arraybackslash}p{0.8cm}>{\raggedright\arraybackslash}p{12cm}>{\centering\arraybackslash}p{1.3cm}}
\caption{Prompts ordered by performance on all $5,000$ examples.} \\
\toprule
\textbf{Rank} & \textbf{Prompt} & \textbf{Test Error} \\
\midrule
\endhead

\midrule
\multicolumn{3}{r}{\small\textit{Continued on next page}} \\
\endfoot

\bottomrule
\endlastfoot

1 & Analyze the image carefully to identify and count all visible shapes from the specified list, distinguishing them from any background patterns or non-relevant objects. First, scan the entire image systematically, noting each shape's outline and geometric properties. Then, categorize each detected shape by type (e.g., square vs. rectangle) while verifying no overlaps or partial occlusions are causing miscounting. Finally, tally the total number of valid shapes, ensuring no duplicates or background artifacts are included. Provide the final count as a numeric value between 1 and 12 enclosed in \verb|\mbox{}| & 0.2174 \\

2 & The image provided to you contains various shapes. How many shapes are present? Please show your work and put your final answer inside the brackets of \verb|\mbox{}|. Think carefully, this problem can be tricky! & 0.2440 \\

3 & Count the geometric shapes (e.g., circles, triangles, squares). Put your answer inside the brackets of \verb|\mbox{}|. & 0.2504 \\

4 & The image provided to you contains various shapes. How many shapes are present? Please show your work and put your final answer inside the brackets of \verb|\mbox{}| & 0.2518 \\

5 & The image provided to you contains various shapes. How many shapes are present? Please show your work and put your final answer as \verb|\mbox{#number of shapes}| & 0.2544 \\

6 & The image has shapes. How many? You need to use this: \verb|\mbox{#number of shapes}|. & 0.2728 \\

7 & Count the shapes, put your answer inside the brackets of \verb|\mbox{}|. & 0.2734 \\

8 & Look at the image and count the shapes. Once you've finished, provide your answer in this format: \verb|\mbox{}|. & 0.2822 \\

9 & Please determine how many shapes are present in the image provided to you. Note that the background should not be included in the count. Once you've calculated the total, express your answer in this format: \verb|\mbox{#total shapes}|. & 0.3004 \\

10 & You will receive an image that displays a diverse array of geometric shapes, which may include hexagons, ovals, triangles, and rectangles. Your task is to carefully and meticulously count the total number of distinct shapes that appear within the image. It is important to note that these shapes can vary significantly in size, be located in different areas of the image, and may also differ in color and orientation. Additionally, be aware that the background of the image might feature a gradient pattern that visually resembles a circular shape. This gradient, however, should not be counted as one of the shapes for the purposes of your analysis. Once you have completed your count and are confident in your assessment, please ensure that you report your final answer explicitly by placing it inside the brackets of the \verb|\mbox{}| command. & 0.3098 \\

11 & The image provided to you contains various shapes. How many shapes are present? You MUST be in the following format: \verb|\mbox{#number of shapes}| & 0.3140 \\

12 & in the pic count the shapes. put the answer inside the brackets of \verb|\mbox{}| & 0.3146 \\

13 & The image provided to you contains various shapes. How many shapes are present? You ABSOLUTELY MUST be in the following format: \verb|\mbox{#number of shapes}| & 0.3180 \\

14 & Please count the number of shapes that appear in the following image. Be aware that the background should not be counted in the calculations. Once you have your final answer, put your answer in the following format: \verb|\mbox{#number of shapes}|. & 0.3222 \\

15 & Count the shapes. After counting the total number of shapes, recount the number to make sure that the final answer is correct, and then put your answer inside the brackets of \verb|\mbox{}|. & 0.3360 \\

16 & Count the shapes. After counting the total number of shapes, double-check your answer for accuracy, and then put your answer inside the brackets of \verb|\mbox{}|. & 0.3374 \\

17 & Count the shapes seen in the image. Put your answer inside the brackets of \verb|\mbox{}|. For example, if there are 3 circles, 2 triangles, and 1 square, the total count is 6. Put your answer as \verb|\mbox{6}|. & 0.3442 \\

18 & Count the shapes. First, identify all individual shapes (e.g., circle, square, triangle). Second, count each shape and add them together. Third, put your answer for the total count inside the brackets of \verb|\mbox{}|. & 0.3504 \\

19 & The image provided contains various shapes such as triangles, squares, and circles. Add the total number of shapes that can be seen in the picture. After adding the number of shapes, include your final answer in the following format: \verb|\mbox{}|. & 0.3774 \\

20 & Analyze the image carefully. Identify and count all distinct shapes present in the image, considering their individual characteristics, such as size, color, and orientation. Be meticulous in ensuring that each unique shape is accounted for. Provide your answer inside the brackets of \verb|\mbox{}|. & 0.3914 \\

21 & The image provided to you contains several shapes, including triangles, squares, and circles. Count all the visible shapes in the picture and calculate the total number. Once you have the total, present your final answer in this format: \verb|\mbox{}|. & 0.4208 \\

22 & Identify and count all the shapes in the image (e.g., circles, squares, triangles). Begin by identifying each distinct shape. Then, tally the total number of shapes. Finally, place your total count inside the brackets of \verb|\mbox{}|. & 0.4358 \\

23 & Count the shapes that appear in the image. List all the shapes you observe (e.g., circle, triangle, square), count them to determine the total, and then put the answer inside the brackets of \verb|\mbox{}|. & 0.4540 \\

24 & Count the number of distinct shapes in the image with care and caution. There are 1-12 shapes generated. You will be graded on the L1 distance of your answer to the ground truth. Report your answers in the brackets of \verb|\mbox{}|. & 0.4546 \\

25 & Count only the shapes (e.g., circles, triangles, squares) and exclude any lines or patterns. Put your answer inside the brackets of \verb|\mbox{}|. & 0.4694 \\

26 & Count the number of distinct shapes in the image with care and caution. There are 1-12 shapes generated. If you are unable to count the number of shapes from the image, choose the final number to be somewhere in the middle. Report your answers in the brackets of \verb|\mbox{}|. & 0.4914 \\

27 & Analyze the attached image and count the number of distinct shapes present. The image may contain between 1 and 12 shapes, including rectangles, squares, circles, ovals, hexagons, pentagons, and stars. Be sure to disregard any background elements or patterns that are not one of the specified shapes. Provide a breakdown of the count for each shape type (e.g., ``3 circles, 2 squares, 1 star'') if possible, and then provide the total number of shapes. Finally, record the total number of shapes within the curly braces of the LaTeX command \verb|\mbox{}|. & 0.5020 \\

28 & Please recognize and count the different shapes in the image, and categorize them by type. Put the final number of shapes in the image in the format: \verb|\mbox{}|. & 0.5060 \\

29 & Count all complete, non-overlapping shapes visible and put your answer inside the brackets of \verb|\mbox{}|. & 0.5094 \\

30 & Thoroughly inspect the image and identify all the distinct shapes. List the shapes by their type (e.g., circles, squares, triangles), then count the total number of shapes. Finally, present your total count in the following format: \verb|\mbox{}|. & 0.5232 \\

31 & Please identify and count the different shapes in the image, and list them by type. Put the final number of shapes in the image in the format: \verb|\mbox{}|. & 0.5234 \\

32 & Count the number of distinct shapes in the image with care and caution. There are 1-12 shapes generated. If you are unable to count the number of shapes from the image, choose the final number to be 6. Report your answers in the brackets of \verb|\mbox{}|. & 0.5370 \\

33 & Carefully examine the image and identify all the distinct shapes. List the shapes by their type (e.g., circles, squares, triangles), then count the total number of shapes. Finally, present your total count in the following format: \verb|\mbox{}|. & 0.5590 \\

34 & Examine this image carefully and help me count the total number of distinct geometric shapes present. First, scan the image systematically from left to right, top to bottom, identifying each instance of rectangles, squares, circles, ovals, hexagons, pentagons, and stars. Ignore any background patterns or decorative elements that aren't complete shapes. For similar shapes that overlap or touch, count them as separate shapes if you can clearly distinguish their individual boundaries. Keep a running tally for each shape type you encounter, then sum them up for the total count. Express your final answer using the format \verb|\mbox{n}| where n is the total number of shapes you've counted. Before giving your final answer, double-check your count to ensure you haven't missed any shapes or counted any twice. & 0.6262 \\

35 & Count the number of distinct shapes in the image with care and caution. There are 1-6 shapes generated. Report your answers in the brackets of \verb|\mbox{}|. & 0.8400 \\

36 & Count the number of distinct shapes in the image with care and caution. There are 1-24 shapes generated. Report your answers in the brackets of \verb|\mbox{}|. & 0.9718 \\

37 & Count the number of distinct shapes in the image with care and caution. There are 1-12 shapes generated. Report your answers in the brackets of \verb|\mbox{}|. & 1.0892 \\

38 & The image you've received contains several geometric shapes of various types (e.g., circles, squares, triangles, etc.). Please identify and count each type of shapes separately. After counting, please show the breakdown of your findings (e.g., how many circles, squares, triangles, etc.). Finally, provide the total number of shapes in the format: \verb|\mbox{#number of shapes}|. Ensure that you answer accounts for each shape present, and please explain how you identified and counted them. & 1.1234 \\

39 & See the image and count shapes; put answer inside \verb|\mbox{}| & 1.1276 \\

40 & Count the number of distinct shapes in the image with care and caution. There are 1-100 shapes generated. Report your answers in the brackets of \verb|\mbox{}|. & 1.6580 \\

\end{longtable}
}

\jmlr{\bibliography{bib.bib}}

\end{document}